%% file: _main.tex
\documentclass[letterpaper, 10 pt, conference]{formatting/ieeeconf}
\IEEEoverridecommandlockouts                              
\let\labelindent\relax
\usepackage{amsfonts}       

\usepackage{amsthm}
\usepackage{mathtools}      
\usepackage{amssymb}        
\usepackage{filecontents}
\usepackage{graphicx}       
\usepackage{marginnote}     
\usepackage{marvosym}       
\usepackage{overpic}        
\usepackage{tabularx}
\usepackage{cite}
\usepackage{color}
\usepackage[linesnumbered,algoruled,boxed,lined]{algorithm2e}
\usepackage[normalem]{ulem}
\usepackage{epstopdf}
\usepackage{graphicx}

\usepackage{epstopdf}
\usepackage{enumitem}
\usepackage[normalem]{ulem}
\usepackage{lipsum}
\usepackage{comment}

\usepackage{diagbox}
\usepackage{multirow}
\usepackage{caption}
\usepackage{subcaption}
\usepackage{CJKutf8}
\usepackage[dvipsnames]{xcolor}

\newif\ifdraft
\draftfalse

\newif\ifoc 
\ocfalse

\newif\ifarxiv 
\arxivtrue

\usepackage[textsize=footnotesize]{todonotes}
\ifdraft
\usepackage[paperheight=11in,paperwidth=9.5in,
			left=1.25in,right=1.25in,
			top=0.75in,bottom=0.75in,
			heightrounded,marginparwidth=1.2in,
			marginparsep=0.05in]{geometry}
\usepackage{xcolor}
\usepackage{xargs} 
\newcommandx{\nt}[2][1=]{\todo[linecolor=red,
			backgroundcolor=red!10,bordercolor=red,#1]{#2}}
\newcommandx{\jy}[2][1=]{\todo[linecolor=green,
			backgroundcolor=green!10,bordercolor=green,#1]{JY: #2}}
\newcommandx{\sw}[2][1=]{\todo[linecolor=blue,
			backgroundcolor=blue!10,bordercolor=blue,#1]{SW: #2}}
\else
\newcommand{\nt}[1]{{}}
\newcommand{\jy}[1]{{}}
\newcommand{\sw}[1]{{}}
\fi

\newif\iftwocolumn
\twocolumntrue

\newtheorem{problem}{Problem}[section]
\newtheorem{proposition}{Proposition}[section]
\newtheorem{lemma}{Lemma}[section]

\newtheorem{theorem}{Theorem}[section]
\theoremstyle{definition}

\theoremstyle{remark}
\newtheorem*{remark}{Remark}

\SetKwProg{Fn}{Function}{}{}
\SetKwComment{Comment}{$\triangleright$\ }{}

\makeatletter
\def\subsubsection{\@startsection{subsubsection}
                                 {3}
                                 {\z@ \hspace*{1mm}}
                                 {0ex plus 0.1ex minus 0.1ex}
                                 {0ex}
                                 {\normalfont\normalsize\itshape}}
\makeatother

\def\W{\mathcal W}

\def\osgt{\textsc{OSG${}_{\mathrm{2D}}$}\xspace}

\def\opg{\textsc{OPG}\xspace}
\def\opgt{\textsc{OPG${}_{\mathrm{2D}}$}\xspace}
\def\orgt{\textsc{{ORG}${}_{\mathrm{2D}}$}\xspace}

\title{
Optimally Guarding Perimeters and Regions with Mobile Range Sensors
}
\author{
Si Wei Feng  \and Jingjin Yu
 \thanks{
 S.-W. Feng and J. Yu are with the Department of Computer Science, 
 Rutgers, the State University of New Jersey, Piscataway, NJ, 
 USA. E-Mails: \{{\tt siwei.feng, jingjin.yu}\}\hspace*{.25em}
 \MVAt \hspace*{.25em}rutgers.edu. This work is supported by NSF 
 awards IIS-1734419 and IIS-1845888. 
}
}

\begin{document}

\maketitle
\thispagestyle{empty}
\pagestyle{empty}



\begin{abstract}
We investigate the problem of using mobile robots equipped with 2D range 
sensors to optimally guard perimeters or regions.
Given a bounded set in $\mathbb{R}^2$ to be guarded, and $k$ mobile sensors 
where the $i$-th sensor can cover a circular region with a variable radius 
$r_i$, we seek the optimal strategy to deploy the $k$ sensors to fully 
cover the set such that $\max r_i$ is minimized. 
On the side of computational complexity, we show that computing a 
$1.152$-optimal solution for guarding a perimeter or a region is NP-hard
even when the set is a simple polygon or the boundary of a simple polygon, 
i.e., the problem is hard to approximate.
The hardness result on perimeter guarding holds when each sensor 
may guard at most two disjoint perimeter segments. 
%
On the side of computational methods, for the guarding perimeters, 
we develop a fully polynomial time approximation scheme (FPTAS) for the 
special setting where each sensor may only guard a single continuous 
perimeter segment, suggesting that the aforementioned hard-to-approximate 
result on the two-disjoint-segment sensing model is tight. 
For the general problem, we first describe a polynomial-time 
$(2+\varepsilon)$-approximation algorithm as an upper bound, applicable to 
both perimeter guarding and region guarding. 
This is followed by a high-performance integer linear programming (ILP) 
based method that computes near-optimal solutions. 
Thorough computational benchmarks as well as evaluation on potential 
application scenarios demonstrate the effectiveness of these algorithmic 
solutions. 
\end{abstract}

\section{Introduction}\label{sec:intro}
\input{texs/01-intro}

\section{Preliminaries}\label{sec:problem}
\input{texs/02-problem}

\section{Intractability of Approximate Optimal Guarding of Simple Polygon}\label{sec:complexity}
\input{texs/03-complexity}

\section{Effective Algorithmic Solutions for \osgt}\label{sec:algo}
\input{texs/04-algo}

\section{Evaluation and Application Scenarios}\label{sec:expr}
\input{texs/05-evaluation}

\section{Conclusion and Discussions}\label{sec:conc}
\input{texs/06-conclusion}

\bibliographystyle{formatting/IEEEtran}
\bibliography{bib/bib,bib/bib2}
\end{document}

%% file: texs/01-intro.tex
In this paper, we consider the problem of using mobile robots equipped 
with range sensors to guard (1D) perimeters or (2D) regions. Given 
a bounded polygonal one- or two-dimensional set to be secured, 
and $k$ mobile robots where robot $i$'s sensor covers a circular region of 
radius $r_i$, we seek a deployment of the robots so that $\max r_i$ is 
minimized. That is, we would like to minimize the maximum single-sensor 
coverage across all sensors. We denote this multi-sensor coverage problem 
under the umbrella term {\em optimal set guarding with 2D sensors}, or 
\osgt.\footnote{The subscript is placed here to distinguish our setup 
from the \opg problem studied in \cite{FenHanGaoYuRSS19}, which assumes 
a 1D sensing model.} The specific problem for guarding
perimeters (resp., regions) is denoted as {\em optimal perimeter 
(resp., region) guarding with 2D sensors}, abbreviated as \opgt (resp., 
\orgt). Beside direct relevance to sensing, surveillance, and monitoring 
applications using mobile sensors \cite{batalin2002spreading,
cortes2004coverage,FenHanGaoYuRSS19}, \osgt applies to other robotics
related problem domains, e.g., the deployment of ad-hoc mobile wireless 
networks \cite{correll2009ad,gil2012communication}, in which case an 
optimal solution to \osgt provides a lower bound on the guaranteed network 
strength over the targeted 2D region. 

\begin{figure}[ht]
    \centering
		\vspace*{2mm}
    \includegraphics[width=\columnwidth]{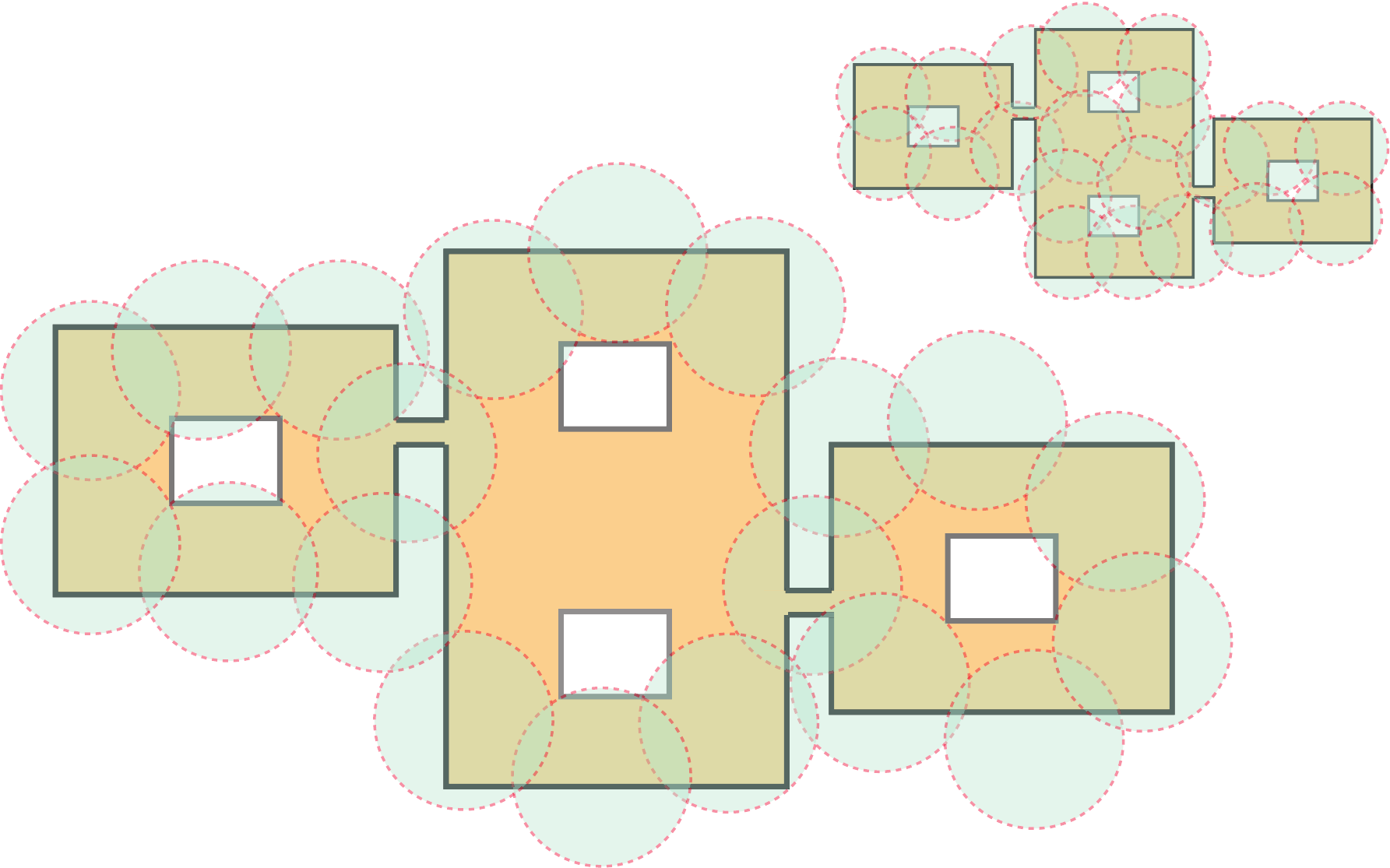}
		\vspace*{1mm}
    \caption{An illustration of the \osgt setup and sample solutions.
		[center] The background shows the footprint of a building, e.g., an 
		apartment  complex. Scenarios may arise that a dangerous criminal 
		might be hiding in the building and we would like to closely monitor 
		the outer boundary of the building.	For the setting, the shaded 
		discs provide a	near-optimal cover with minimum radii for $20$ 
		mobile sensors that fully encloses the outer perimeter, computed 
		using algorithms presented in this work with optimality guarantees. 
		[upper right] A near-optimal solution for guarding the interior of
		the building footprint minus the four holes.}
    \label{fig:example0}
\end{figure}
As a summary of the study, on the side of computational complexity, we 
establish that \opgt is hard to approximate within a factor of 
$1.152$
even when the perimeter is a simple closed polygonal chain whose length
is bounded by the input size, through a reduction from 
vertex cover on planar $3$-regular graphs.
A unique property of our reduction is that it shows the inapproximability gap
remains when each sensor can cover at most two disjoint perimeter 
segments. 
The proof also shows that \orgt is at least as hard to approximate. 
Therefore, no polynomial time algorithm may exist that solves \osgt to 
better than the $1.152$-optimal lower bound, unless P$\,=\,$NP. 
On the algorithmic side, we begin by providing an efficient $(1+\varepsilon)$
approximation algorithm
for a specific class of \opgt problems in which each mobile 
sensor must cover a continuous perimeter segment. This implies that the 
aforementioned inapproximability result on \opgt under the 
two-disjoint-segment sensing model is tight. 
For the general \osgt problem, we first describe a polynomial time 
$(2+\varepsilon)$ approximation algorithm as a reasonable approximability 
upper bound. 
Then, an integer linear programming (ILP) model is devised that allows 
the fast computation of highly optimal solutions for fairly large 
problem instances. 
%
Results described in this paragraph, together with the introduction 
of \osgt as a practical multi-robot deployment problem focusing on global 
optimality, constitute the main contributions of this work. 

As an intermediate result toward showing the hardness of the simple polygon
coverage problem, we also supply a hardness proof of vertex cover 
on planar bridgeless\footnote{That is, the deletion of any edge does not disconnect 
the graph.} $3$-regular graphs, which may be of independent interest. 

\textbf{Related work}. Our work on optimal perimeter and region guarding 
draws inspiration from a long line of multi-robot coverage planning and 
control research, e.g., \cite{cortes2004coverage,martinez2007motion,
schwager2009optimal,pavone2009equitable,schwager2009decentralized,
pierson2017adapting}. 
In an influential body of work on coverage control \cite{cortes2004coverage,
martinez2007motion}, a gradient based iterative method is shown to drive 
one or multiple mobile sensors to a locally optimal configuration with 
convergence guarantees. 
Whereas \cite{cortes2004coverage,martinez2007motion} assume that the 
distribution of sensory information is available {\em a priori}, it is 
shown that such information can be effectively learned 
\cite{schwager2009decentralized}. 
Subsequently, the control method is further extended to allow the 
coverage of non-convex and disjoint 2D domains \cite{schwager2009optimal} 
and to work for mobile robots with varying sensing or actuation capabilities
\cite{pierson2017adapting}. 
In contrast to these control-based approaches, which produce iterative 
locally optimal solutions, \osgt emphasizes the direct computation of 
globally optimal deployment solutions and supports arbitrarily shaped
bounded (1D) perimeters and (2D) regions.

Recently, the problems of globally optimally covering perimeters using 
one-dimensional sensors have been studied in much detail 
\cite{FenHanGaoYuRSS19,FenYu2020RAL}. It is shown that when the sensors 
are homogeneous, the optimal deployment of sensors can be computed 
very efficiently, even for highly complex perimeters \cite{FenHanGaoYuRSS19}.
On the other hand, the problem becomes immediately intractable, sometimes
strongly NP-hard, when sensors are heterogeneous \cite{FenYu2020RAL}. 
Our research is distinct from \cite{FenHanGaoYuRSS19,FenYu2020RAL} in that 
we employ a (two-dimensional) range sensing model and work on the coverage 
of both perimeters and regions, which has much broader applicability. 

As pointed out in \cite{cortes2004coverage,schwager2009decentralized}, 
distributed sensor coverage, as well as \osgt, has roots in the study 
of the facility location optimization problem 
\cite{weber1929theory,drezner1995facility}, which examines the selection 
of facility (e.g., warehouses) locations that minimize the cost of delivery 
of supplies to spatially distributed customers. In theoretical computer 
science and operations research, these are known as the $k$-center, 
$k$-means, and $k$-median clustering problems \cite{har2011geometric}, 
the differences among which are induced by the cost structure. Our 
investigation of \osgt benefits from the vast literature on 
the study of $k$-center clustering and related problems, e.g., 
\cite{feder1988optimal,hochbaum1985best,gonzalez1985clustering,daskin2000new,shamos1975closest}.
These clustering problems are in turn related to packing 
\cite{hales2005proof}, tiling \cite{thue1910dichteste}, and the 
well-studied art gallery problems \cite{o1987art,shermer1992recent}.

\textbf{Organization}. The rest of the paper is organized as follows. 
In Section~\ref{sec:problem}, we introduce the \osgt formulation. 
Section~\ref{sec:complexity} is devoted to establishing that \osgt is 
hard to approximate to better than $1.152$-optimal, providing a theoretical
lower bound. In Section~\ref{sec:algo}, focusing on the upper bound, 
we describe algorithms that for \osgt and the special \opgt variant 
where a sensor is allowed to cover a continuous perimeter segment.
In Section~\ref{sec:expr}, we benchmark the algorithms and illustrate two 
potential applications. We discuss and conclude the work in Section~\ref{sec:conc}.

\jy{Q: in the perimeter guarding problem, what if the perimeters are paths and the goal is to just touch all paths at least at one point? Maybe something interesting to explore.} 

\jy{Q: what about non-circular coverage region?}

\jy{Q: we may explore the setting that restricts the sensor deployment location. E.g., 
we may only allow sensors to be on the boundary, or within some distance away from 
the boundary, and so on.}

%% file: texs/02-problem.tex
 Let $\mathcal{W}\subset \mathbb{R}^2$ be a 
polygonal workspace, which may contain one or multiple connected components. 
A {\em critical subset} of $\mathcal{W}$ needs to be guarded by $k$ 
indistinguishable point guards with range sensing capabilities. For 
example, the workspace may be a forest reserve and the critical subset 
may be its boundary. Or, the workspace may be a high-security facility, 
e.g., a prison, and the critical subset the prison yard. The $i$th 
guard, $1\le i\le k$, located at $c_i \in \mathbb{R}^2$, can monitor a circular 
area of radius $r_i$ centered at $c_i$ with $r_i$ being a variable. For 
example, the guard may be a watchtower equipped with a vision sensor 
that can detect intruders. As the watchtower's altitude increases, 
its sensing range also increases; but its monitoring quality will 
decrease at the same time due to resolution loss. In this study, we seek 
to compute the optimal strategy to deploy these $k$ guards so that the 
required sensing range, $\max_i r_i$, could be minimized. 

More formally, we model a connected component of $\W$ as some 2D polygonal region
containing zero or more simple polygonal obstacles. 
For a bounded set $D \subset \mathbb{R}^2$, we define
\begin{align*}
size(k, D) = \min_{c_1, \dots, c_k\in \mathbb{R}^2}\ \max_{p\in D}\ \min_{1\leq i\leq k} \|c_i - p\|_2  
\end{align*}
and use $B(c, r)$ to denote the disc of radius $r$ centered at a point $c 
\in \mathbb{R}^2$ (the definition of $size(k, D)$ is used extensively in later 
sections). 
Intuitively, $size(k,D)$ represents the minimum radius needed such that there
exisits $k$ circles with radius $size(k, D)$ that can cover the 2D bounded region $D$ entirely.
The main problem studied in this work is:

\begin{problem}[Optimal Set Guarding with 2D Sensors]
Given a polygonal workspace $\W \subset \mathbb{R}^2$, let $D \subset \W$ be a 
critical subset to be guarded by $k$ robots each with a variable coverage radius of $r$. Find 
the smallest $r$ and corresponding robot locations $c_1, \dots, c_k\in \mathbb{R}^2$, such that $D \subset \cup_i B(c_i, r)$.
\end{problem}

For making accurate statements about computational complexity, we make
the assumption that the length of $\partial \W$ is bounded by a polynomial with respect 
to the complexity of $\W$, (i.e. the number of vertices of the polygon). 

For convenience, we give specific names to these optimal guarding 
problems based critical subset types. If the critical 
subset belongs to $\partial \W$, we denote the problem as {\em optimal 
perimeter guarding with 2D sensors} or \opgt. If the critical subset 
is $\W$, we denote the problem as {\em optimal region guarding with 
2D sensors} or \orgt. When there is no need to distinguish, the problem 
is denoted as {\em optimal set guarding with 2D sensors} (\osgt). 

As an example, to guard the boundary of a plus-shaped polygon with $5$ 
robots, an optimal solution could be Fig.~\ref{fig:example} where the 
inner circle covers $4$ disconnected boundary segments, such pattern 
in the optimal solution also renders \opgt much more difficult than 
the simplified 1D sensing model studied in \cite{FenHanGaoYuRSS19} 
(indeed, \opgt becomes hard to approximate, as will be shown shortly). 
The solution is also optimal under the \orgt formulation.
\begin{figure}[ht]
    \centering
		\vspace*{3mm}
    \includegraphics[scale=0.35]{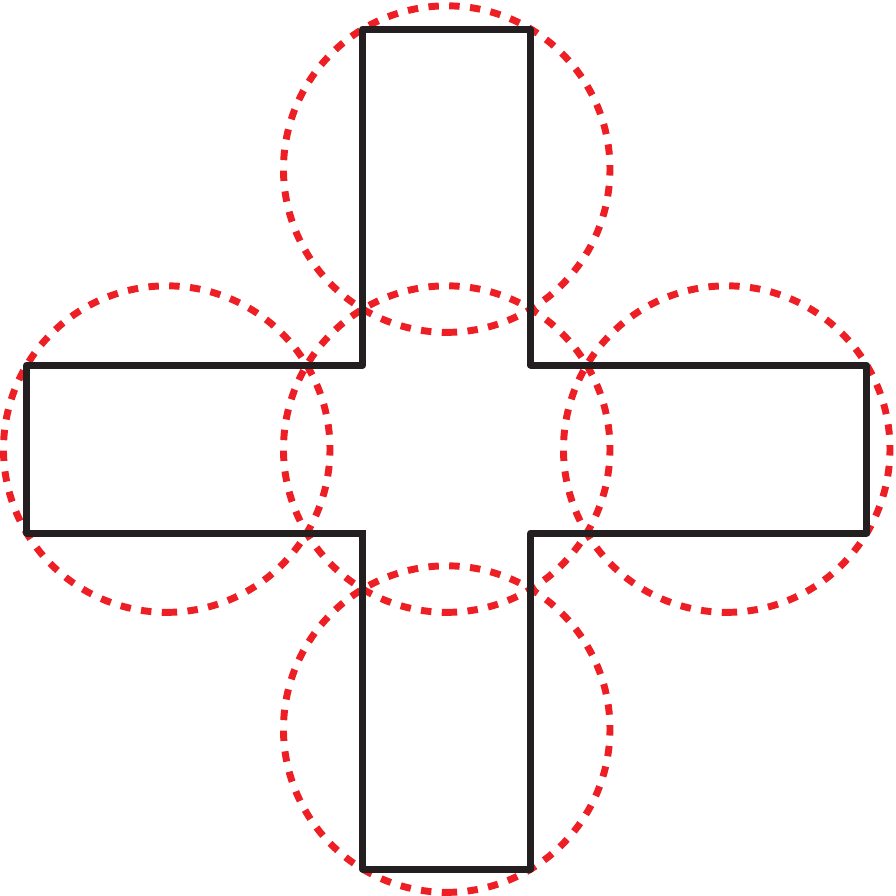}
		\vspace*{1.5mm}
    \caption{An example showing an optimal solution of using five discs
		to cover the plus-shaped polygon. The solution is optimal for both 
		\opgt and \orgt formulations.}
    \label{fig:example}
\end{figure}

%% file: texs/03-complexity.tex
In this section, we prove that \osgt with the set being a simple polygon is 
strongly NP-hard to approximate within a factor of $\alpha \approx 1.152$, through 
a sequence of auxiliary NP-hardness results. 
First, in Section~\ref{subsec:3regular}, we prove an intermediate result 
that the vertex cover problem is NP-complete on planar bridgeless 
$3$-regular graphs.
Next, in Section~\ref{complexity:3netcomp}, starting from a planar bridgeless 
$3$-regular graph, we construct a structure which we call {\em $3$-net} and 
prove the the problem of finding the minimum coverage radius of the 
$3$-net is NP-hard to approximate within $\alpha$. 
Then, in Section~\ref{subsec:osgthard}, we apply a straightforward 
reduction to transform the $3$-net into a simple polygon to complete
the hard-to-approximate proof for \osgt for a simple polygon.

We then further show the inapproximability of the special \opgt setup 
when each robot can only guard at most two disjoint perimeter segments 
(Section~\ref{subsec:2-seghard}), contrasting the FPTAS for the special 
\opgt setup when each robot can only guard a continuous perimeter 
segment in Section \ref{subsec:singleseg}.

\subsection{Vertex Cover on Planar Bridgeless $3$-Regular Graph}
\input{texs/03-complexity4.tex}\label{subsec:3regular}
\subsection{Hardness on Optimally Guarding A $3$-Net}\label{complexity:3netcomp}
\input{texs/03-complexity1.tex}
\subsection{From $3$-Net to A Simple Polygon}\label{subsec:osgthard}
\input{texs/03-complexity2.tex}
\subsection{\opgt with Sensor Guarding Limitations}\label{subsec:2-seghard}
\input{texs/03-complexity3.tex}

%% file: texs/03-complexity4.tex

Our reduction uses the hardness result on the vertex cover problem for planar 
graphs with maximum degree $3$ \cite{garey1977rectilinear}. Such a vertex cover 
problem can be fully specified with a 2-tuple $(G, k)$ where $G = (V, E)$ is a 
planar graph with max degree $3$ and $k$ is an integer specifying the allowed 
number of vertices in a vertex cover. We note that the result has been 
suggested implicitly in \cite{mohar2001face}; we provide an explicit account 
with a simple proof. 

\begin{lemma}
Vertex cover on planar bridgeless $3$-regular graph is NP-complete.
\end{lemma}
\begin{proof}
For a given planar graph $G$ with max degree $3$ and an integer $k$,
we construct a planar bridgeless $3$-regular graph $G''$ and provide an 
integer $k''$ such that $G$ has a vertex cover of size $k$ if and only 
if $G''$ has a vertex cover of size $k''$.

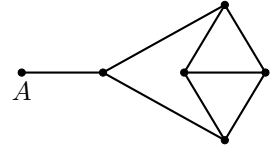
\begin{figure}[ht]
    \vspace*{2mm}
    \centering
    \begin{tikzpicture}[scale = .9]
        \draw[black, thick] (-0.8, 0) -- (-2,0);
        \draw[black, thick] (-0.8, 0) -- (1,1);
        \draw[black, thick] (0.4, 0) -- (1,1);
        \draw[black, thick] (1.6, 0) -- (1,1);
        \draw[black, thick] (-0.8, 0) -- (1,-1);
        \draw[black, thick] (0.4, 0) -- (1,-1);
        \draw[black, thick] (1.6, 0) -- (1,-1);
        \draw[black, thick] (0.4, 0) -- (1.6,0);
        \filldraw[black] (-0.8,0) circle (1.5pt);
        \filldraw[black] (0.4,0) circle (1.5pt);
        \filldraw[black] (1.6,0) circle (1.5pt);
        \filldraw[black] (-2,0) circle (1.5pt) node[anchor = north] {$A$};
        \filldraw[black] (1,-1) circle (1.5pt);
        \filldraw[black] (1,1) circle (1.5pt);
    \end{tikzpicture}
		\vspace*{2mm}
    \caption{A gadget that can be attached to a degree one or two vertices
		(at the point $A$) in a max degree $3$ graph to make all vertices have
		degree $3$. With each addition of the gadget, we increase the vertex 
		cover by a size of $3$, regardless of whether $A$ is part of a vertex 
		cover.}
    \label{fig:structure-hook}
\end{figure}

The reduction first makes $G$ $3$-regular by attaching (one or two of) the 
gadget shown in Fig.~\ref{fig:structure-hook} to $v \in G$ that are not 
degree $3$. This results in a $3$-regular graph $G'$. For each attached 
gadget, $k$ is bumped up by $3$, i.e., we let $k'$ for $G'$ be $k' = k 
+ 3(3|V(G)| - 2|E(G)|)$. It is straightforward to see that $G$ has a vertex 
cover size of $k$ if and only if $G'$ has a vertex cover size of $k'$.

\begin{figure}[!ht]
    \centering
    \raisebox{10mm}
    {\begin{tikzpicture}
        \draw[black, thick] (-1,0) -- (1,0);
        \draw[black, thick] (-1,0) -- (-1.3,0.6);
        \draw[black, thick] (-1,0) -- (-1.5,-0.6);
        \draw[black, thick] (1,0) -- (1.3,0.6);
        \draw[black, thick] (1,0) -- (1.5,-0.6);
        \draw[black, thick, -latex] (2.3, 0) -> (2.8,0);
        \filldraw[black] (-1.5,0) node[anchor=south] {$G_1'$};
        \filldraw[black] (1.5,0) node[anchor=south] {$G_2'$};
        \filldraw[black] (-1,0) circle (1.5pt) node[anchor=north] {$P$};
        \filldraw[black] (1,0) circle (1.5pt) node[anchor=north] {$Q\,\,$};
    \end{tikzpicture}
    }
    \hfill
    \begin{tikzpicture}
        \coordinate(p1) at (-1.5, 0);
        \coordinate(p2) at (-0.5, 0);
        \coordinate(p3) at (-1, 0.5);
        \coordinate(p4) at (-1, -0.5);
        
        \coordinate(p5) at (-1, 1);
        \coordinate(p6) at (-1, -1);
        
        \coordinate(p7) at (-1.5, 1.4);
        \coordinate(p8) at (-1.5, -1.4);

        \coordinate(q1) at (0.5, 0);
        \coordinate(q2) at (1.5, 0);
        \coordinate(q3) at (1, 0.5);
        \coordinate(q4) at (1, -0.5);

        \coordinate(q5) at (1, 1);
        \coordinate(q6) at (1, -1);
        \coordinate(q7) at (1.5, 1.4);
        \coordinate(q8) at (1.5, -1.4);
        
        \draw[black, thick] (p1) -- (p2);
        \draw[black, thick] (p1) -- (p3);
        \draw[black, thick] (p1) -- (p4);        
        \draw[black, thick] (p2) -- (p3);
        \draw[black, thick] (p2) -- (p4);
        \draw[black, thick] (p6) -- (p4);
        \draw[black, thick] (p5) -- (p3);

        \draw[black, thick] (q1) -- (q2);
        \draw[black, thick] (q1) -- (q3);
        \draw[black, thick] (q1) -- (q4);        
        \draw[black, thick] (q2) -- (q3);
        \draw[black, thick] (q2) -- (q4);
        
        \draw[black, thick] (q6) -- (q4);
        \draw[black, thick] (q5) -- (q3);
        
        \draw[black, thick] (q5) -- (p5);
        \draw[black, thick] (q6) -- (p6);
        
        \filldraw[black] (p1) circle (1.5pt);
        \filldraw[black] (p2) circle (1.5pt);
        \filldraw[black] (p3) circle (1.5pt);
        \filldraw[black] (p4) circle (1.5pt);
        \filldraw[black] (p5) circle (1.5pt) node[anchor=south] {$P'$};
        \filldraw[black] (p6) circle (1.5pt) node[anchor=north] {$P''$};
        
        \filldraw[black] (q1) circle (1.5pt);
        \filldraw[black] (q2) circle (1.5pt);
        \filldraw[black] (q3) circle (1.5pt);
        \filldraw[black] (q4) circle (1.5pt);
        \filldraw[black] (q5) circle (1.5pt) node[anchor=south] {$Q'$};
        \filldraw[black] (q6) circle (1.5pt) node[anchor=north] {$Q''$};
        
        \draw[black, thick] (p5) -- (p7);
        \draw[black, thick] (p6) -- (p8);
        \draw[black, thick] (q5) -- (q7);
        \draw[black, thick] (q6) -- (q8);
        
    \end{tikzpicture}
    \vspace{.5mm}
    \caption{Transformation that removes bridge $PQ$ and does not introduce new bridges.
    The minimum vertex cover number is increased by $6$ after each transformation.}
    \label{fig:rmbridge}
\end{figure}
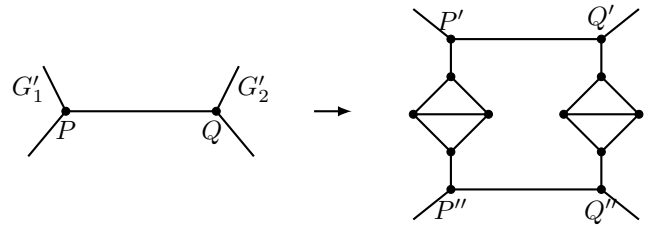

In the second and last step, we remove bridges in $G'$. As in 
Fig.~\ref{fig:rmbridge}, for a bridge $PQ$ that divides $G'$ into $G_1'$ 
(containing $P$) and $G_2'$ (containing $Q$), we split the bridge edge 
$PQ$ using the illustrated transformation, which yields a new graph $G''$
that is planar, bridgeless, and $3$-regular, after all bridges are removed
this way. For each such augmentation, the size of the vertex cover is 
bumped up by $6$. Let $br(G')$ be the number of bridges in $G'$, $G'$ has 
a vertex cover of size  $k'$ if and only if $G''$ has a vertex cover 
of size $k'' = k'+6br(G')$. This completes the proof. 
\end{proof}

%% file: texs/03-complexity1.tex
Starting from a planar cubic graph $G$, we construct a structure that we call 
$3$-net, $T_G$, as follows. 
First, similar to \cite{feder1988optimal}, to embed $G$ into the plane, 
an edge $uw \in E(G)$ is converted to an odd length path $uv_1, v_1v_2, 
\ldots, v_{2m}w$ where $m > 3$ is an integer. We note that $m$ is different
in general for different edges of $G$. 
Denote such a path as $u\cdots w$; each edge along $u\cdots w$ is straight 
and has unit edge length. We also require that each path is nearly straight 
locally. 
%
For a vertex of $G$ with degree $3$, e.g., a vertex $u \in V(G)$ 
neighboring $w, x, y \in V(G)$, we choose proper configurations and lengths for 
paths, $u\cdots w$, $u\cdots x$, and $u\cdots y$ such that
these paths meet at $u$ forming pairwise angles of $2\pi/3$. We denote the 
resulting graph as $G'$, which becomes the {\em backbone} of the 
$3$-net $T_G$. 



From here, a second modification is made which completes the 
construction of $T_G$. In each previously constructed 
path $u\cdots w = uv_1\ldots v_{2m}w$, for each $v_iv_{i+1}$, $1 \le i
\le 2m-1$, we add a line segment of length $\sqrt{3}$ that is 
perpendicular to $v_iv_{i+1}$ such that $v_iv_{i+1}$ and the line 
segment divide each other in the middle. A graphical illustration is
given in Fig.~\ref{fig:path-bar}. 
$G'$ and the bars form the 
$3$-net, which we denote as $T_G$. An example of transforming $K_4$
into a {\em 3-net} is given in Fig.~\ref{fig:3-net}.

\begin{figure}[ht]
    \centering
		\vspace*{1mm}
    \includegraphics[scale=0.3]{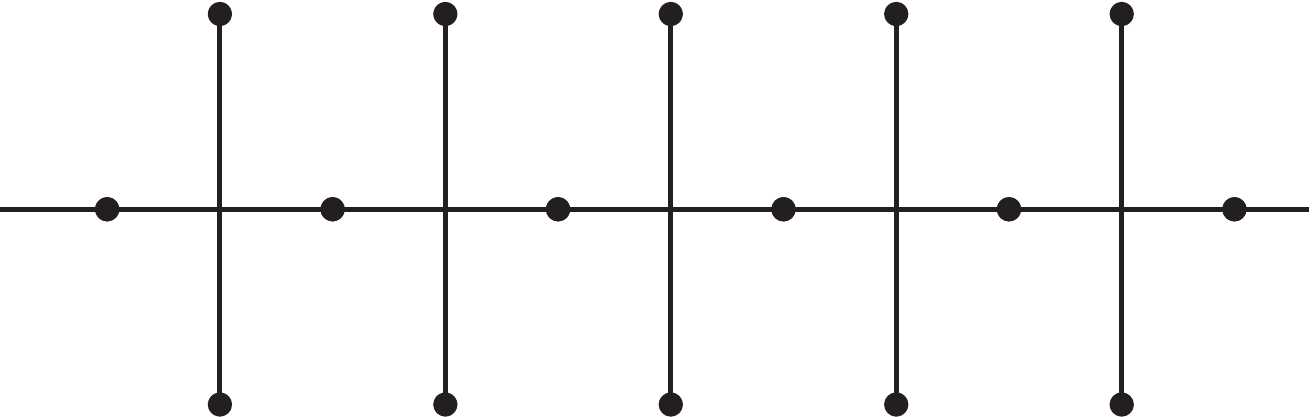}
		\vspace*{3mm}
    \caption{Structure within the odd length path and attached 
		perpendicular ``bars'' with length $\sqrt{3}$. Regarding the 
		representation of such non-integral coordinates in the problem 
    input, 
    we may scale the coordinates to some certain extent and
    round them to integers so that the relative distance between
    each other is precise enough
    for the proof.
    }
		\vspace*{-1mm}
    \label{fig:path-bar}
\end{figure}

\begin{figure}[ht]
    \centering
    \raisebox{14mm}{
    \begin{overpic}[scale=0.2]{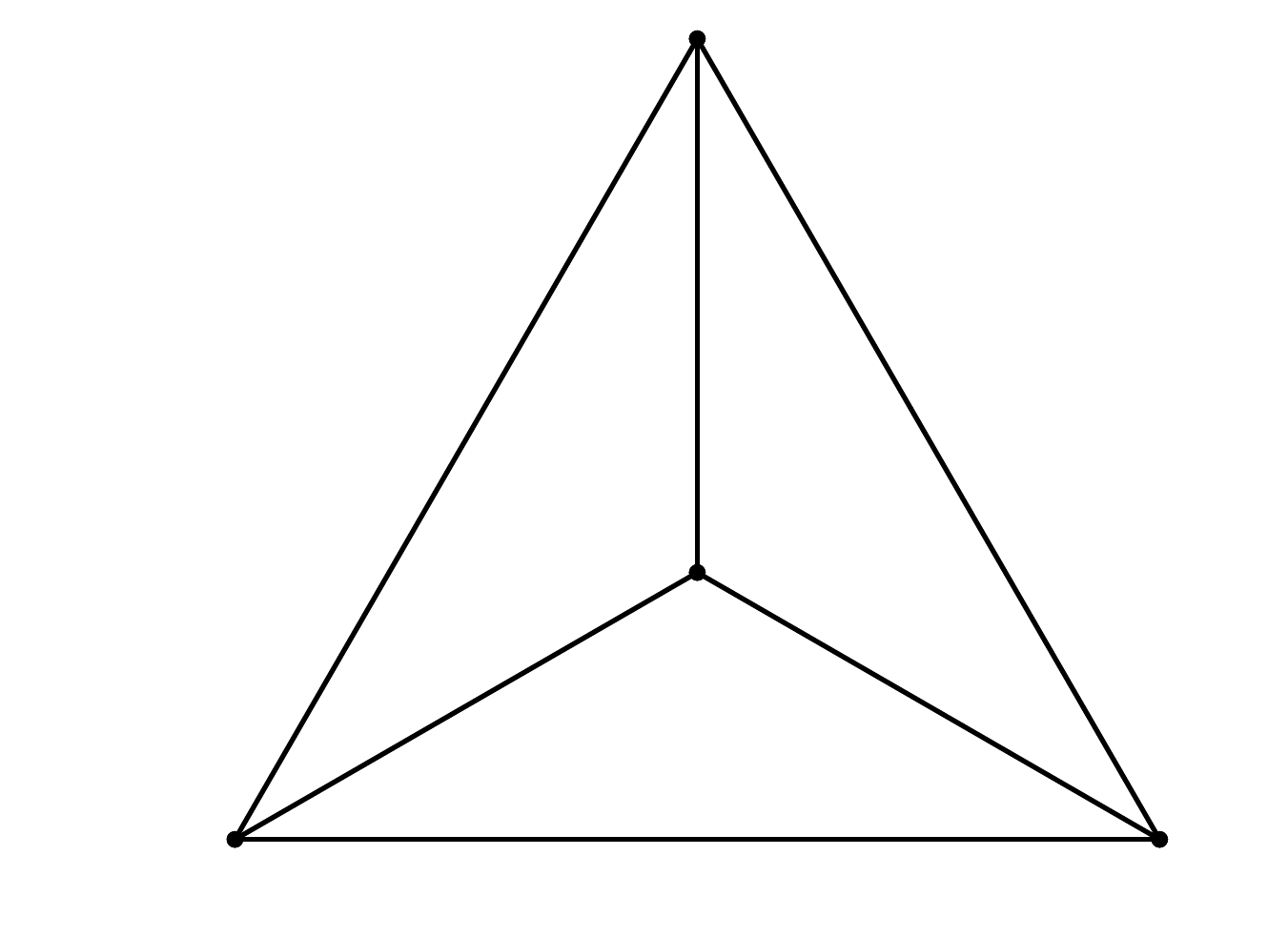}
      \put(115,23){$\to$}
      \end{overpic}
    }
      \begin{overpic}[scale=0.4]{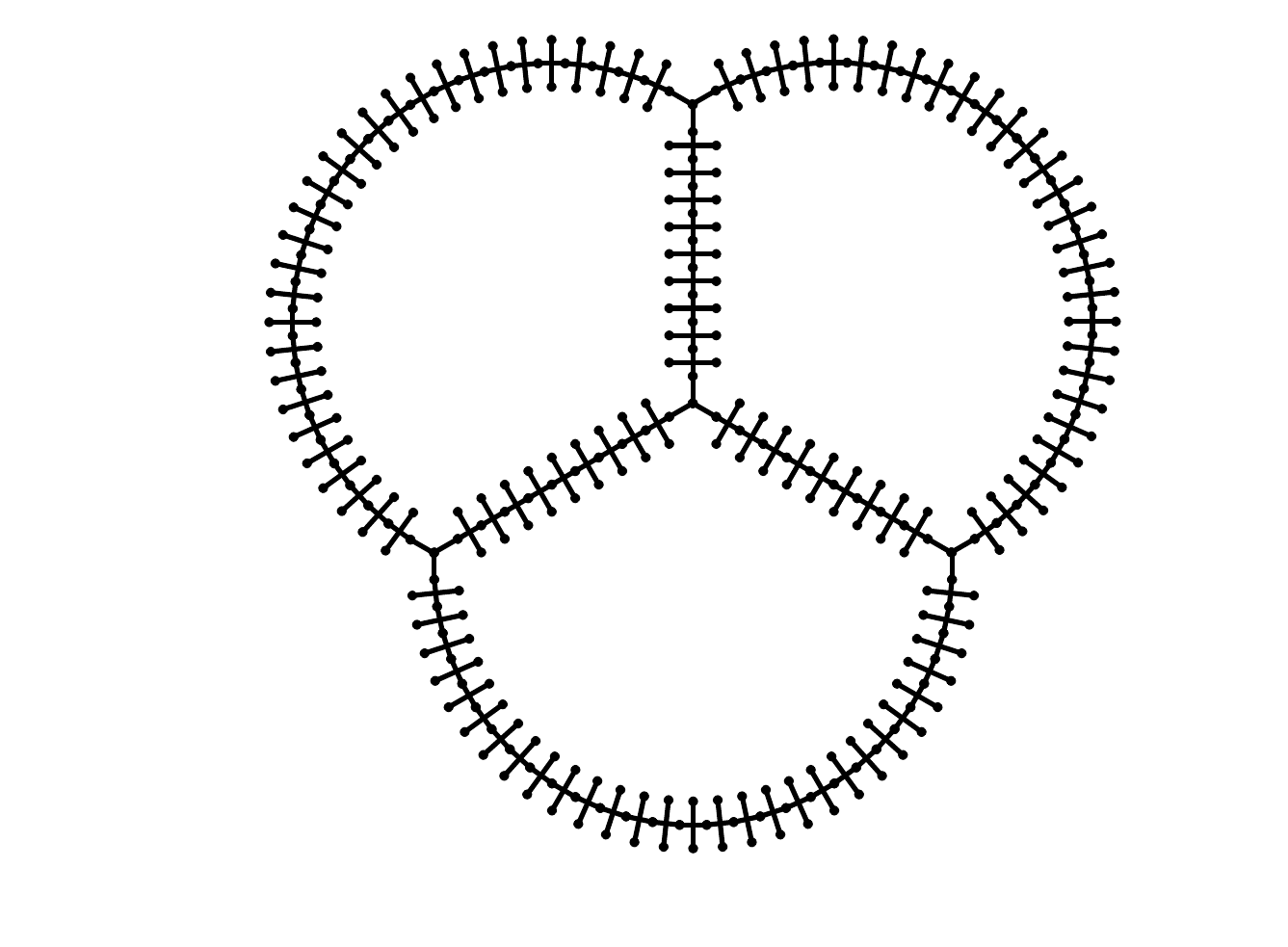}
        \end{overpic}
    \caption{Illustration of a $3$-net obtained from $K_4$, the complete graph on 
		$4$ vertices.}
    \label{fig:3-net}
\end{figure}

Let $L$ be the number of (unit length) edges of $G'$ (i.e., $L = 
\sum_{uv \in E(G)}len(u\cdots w)$). 
\begin{lemma}\label{l:bl}
A planar bridgeless $3$-regular graph $G$ has a vertex cover of size 
$k$ if and only if its transformed 3-net $T_G$ can be covered
by $K = k + (L-|E(G)|)/2$ circles of radius approximately $\alpha = 1.152$.
\end{lemma}
\ifarxiv
{}
\else
\fi

\ifarxiv
{}
\else
We refer the readers to \cite{feng2020optimally} for the lengthy and technical 
proof. 
\fi
From the proof, it is also clear that Lemma~\ref{l:bl} holds for discs 
with radius in $[1, \alpha)$. Thus, approximating $size(T_G, K)$ to less than a 
factor of $\alpha$ will decide whether $G$ has a vertex cover of size $k$, yielding 
the hard-to-approximate result. Also, it can be observed that all lengths are 
polynomial with respect to the problem input size, which implies strongly 
NP-hardness.
\begin{theorem}\label{t:3nethard}
The minimum radius for cover a $3$-net using 
$k$ circular discs 
is strongly NP-hard to approximate within a 
factor of $\alpha \approx 1.152$.
\end{theorem}

%% file: texs/03-complexity2.tex
We proceed to show that \osgt is hard to approximate for a simple polygon 
by converting a $3$-net into one. Along the backbone $G'$ of a $3$-net 
$T_G$, we first expand the line segments by $\delta$ to get a 2D region  
(see Fig.~\ref{fig:door}(a)). We may describe the interior of the resulting 
polygon as 
\vspace*{-1mm}
\[P=\{p\in \mathbb{R}^2\ |\ \min_{q\in T_G}(\|p-q\|_1)\leq \delta/2\}\]
\vspace*{-4mm}

For small enough $\delta$, it's clear that $P$ is a polygon with holes.
Let $K = (({L-|E|}){2}+k)$, it holds that
\vspace*{-1mm}
\begin{align*}
&size(K, T_G) \leq               size(K, P)\leq size(K, T_G) + \delta,\\
&size(K, T_G) \leq      size(K, \partial P)\leq size(K, T_G) + \delta.
\end{align*}
\vspace*{-5mm}

To convert the structure into a simple polygon, we can open ``doors'' of 
width $\delta$ on the structure to get rid of the holes (see 
Fig.~\ref{fig:door}(b)). Each opening removes one hole from $P$. This is 
straightforward to check; we omit the details. 
\begin{figure}[ht]
		\centering
		\vspace*{0mm}
    \begin{small}
    \includegraphics[scale=.3]{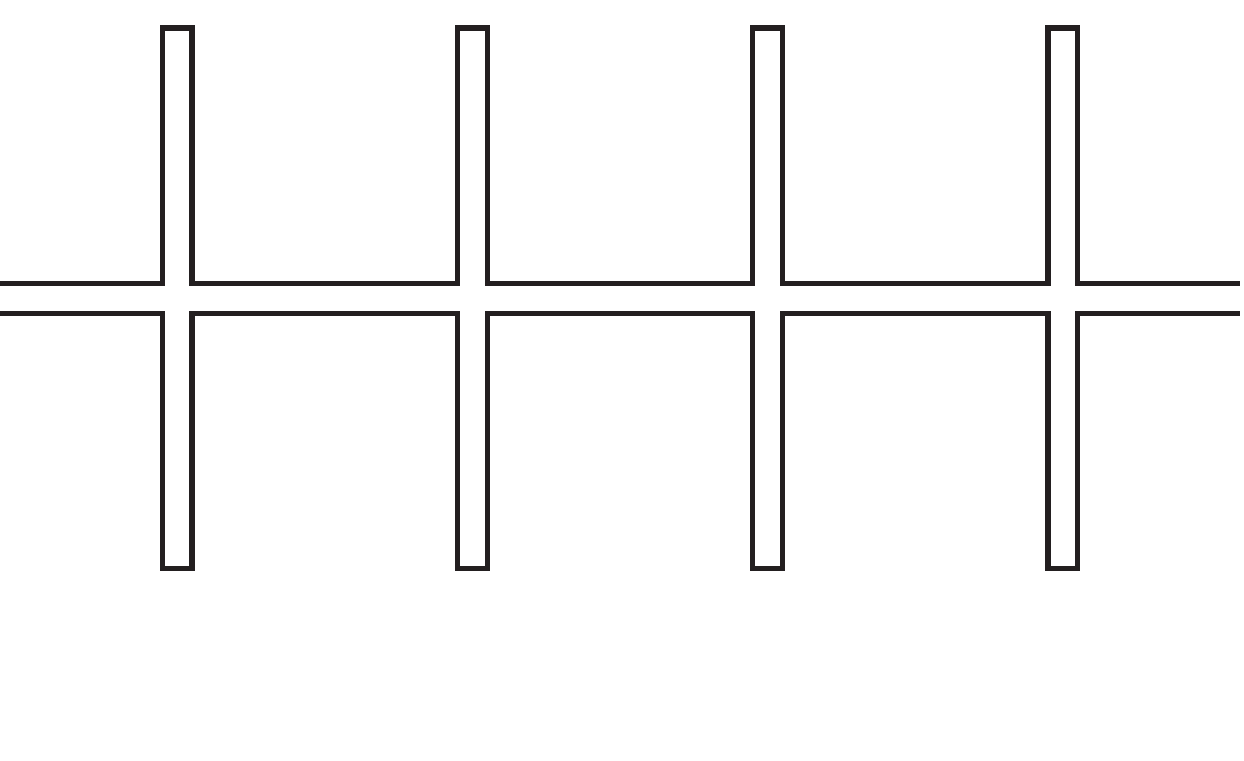} 
    \hfill
    \begin{overpic}[scale=.3]{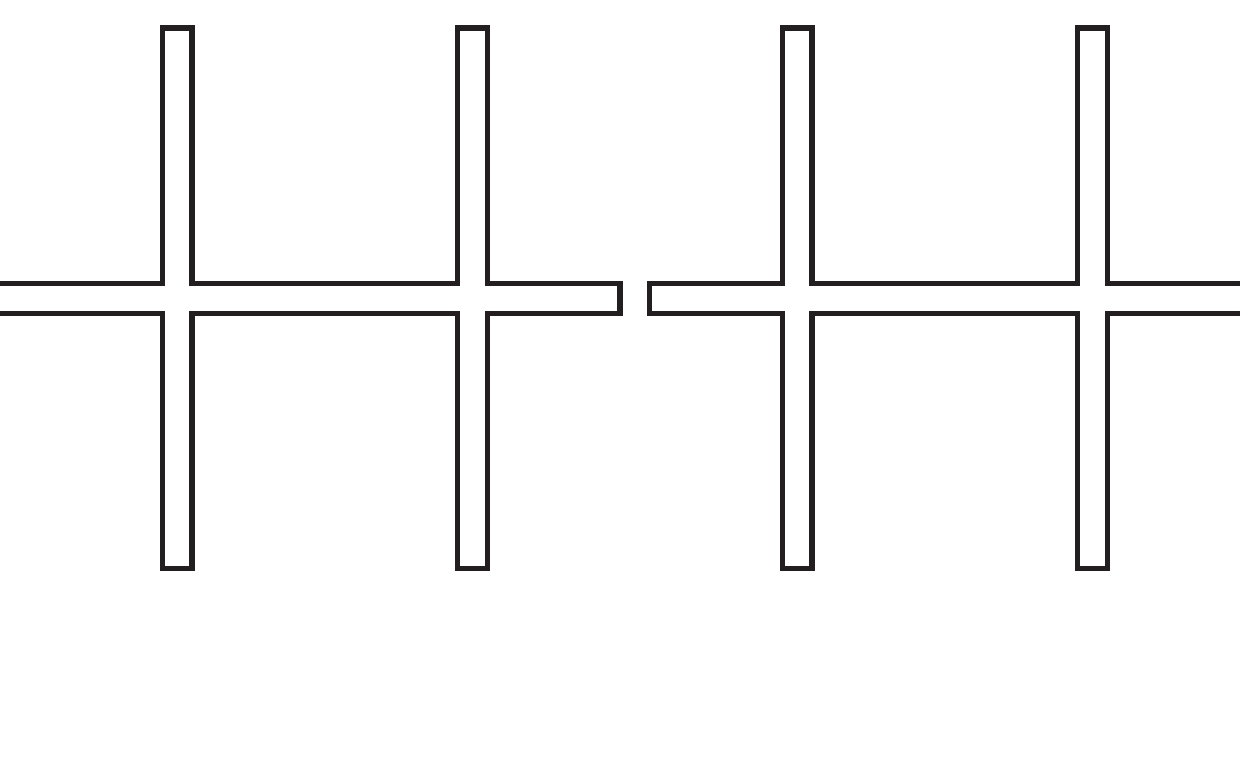}
        \put(50,42){$\delta$}
        \put(-75,3){(a)}
        \put(46,3){(b)}
    \end{overpic}
	  \end{small}
		\vspace*{1mm}
    \caption{(a) A $3$-net $T_G$ maybe readily converted into a simple polygon 
		$P$ with holes by expanding along its backbone. (b) Creating a ``door'' 
		of width $\delta$ will remove one hole from $P$.}
    \label{fig:door}
\end{figure}

Denoting the resulting simple polygon as $P'$, we have
\vspace*{-1mm}
\begin{align*}
&size(K, P) - \delta  \leq               size(K, P')\leq size(K, P),\\
&size(K, \partial P) - \delta  \leq      size(K, \partial P')\leq size(K, \partial P).
\end{align*}
\vspace*{-5mm}

Therefore, both $size(k, P')$ and $size(k, \partial P')$ are between 
$size(k, T_G) - \delta$ and $size(k, T_G) + \delta$. Suppose the \osgt for 
$\partial P'$ or $P'$ has a polynomial approximation algorithm with 
approximation ratio $1.152 - \varepsilon$ where $\varepsilon > 0$, let $\delta = 
\varepsilon/2$, then the optimal guarding problem for the $T_G$ can 
be approximated within 1.152 disobeying the inapproximability gap by 
Theorem~\ref{t:3nethard}. Therefore,

\begin{theorem}\label{t:osgthard}
\osgt is NP-hard and does not admit a polynomial time approximation 
within a factor of $\alpha$ with $\alpha \approx 1.152$, unless 
P$=$NP.
\end{theorem}

%% file: texs/03-complexity3.tex
The inapproximability gap from Theorem~\ref{t:osgthard} prompts us to 
further consider limitations on the setup with the hope that meaningful 
yet more tractable problems may arise. One natural limitation is to 
limit the number of continuous segments a mobile sensor may cover. As 
will be shown in Section~\ref{subsec:singleseg}, if a mobile sensor may 
only guard a single continuous perimeter segment, a 
$(1 + \varepsilon)$-optimal solution can be computed efficiently. 
On the other hand, it turns out that if a sensor can guard up to two 
continuous perimeter segments, \opgt remains hard to approximate. 

\begin{theorem}\label{them:twoconthard}
\opgt of a simple polygon cannot be approximated within $\alpha\approx 1.152$ 
even when each robot can guard no more than two continuous boundary segments,
unless P$=$NP.
\end{theorem}
\begin{proof}
Due to \cite{petersen1891}, every bridgeless $3$-regular graph $G$ has a 
perfect matching. We can obtain such a perfect matching of the $3$-regular 
graph using Edmonds Blossom algorithm in polynomial time\cite{edmonds_1965}. 
Doubling the edges in the perfect matching, we can then obtain a 4-regular 
graph $G'$. 
With each vertex's degree even on $G'$, a Eulerian tour exists
on $G'$ and can be efficiently computed. For the $3$-net $T_G$, we may 
incorporate the bars into the Eulerian tour, corresponding to that
for $G'$, as illustrated in Fig.~\ref{fig:path_exp}.
\begin{figure}[ht]
		\vspace*{-2mm}
    \centering
    \includegraphics[scale=0.3]{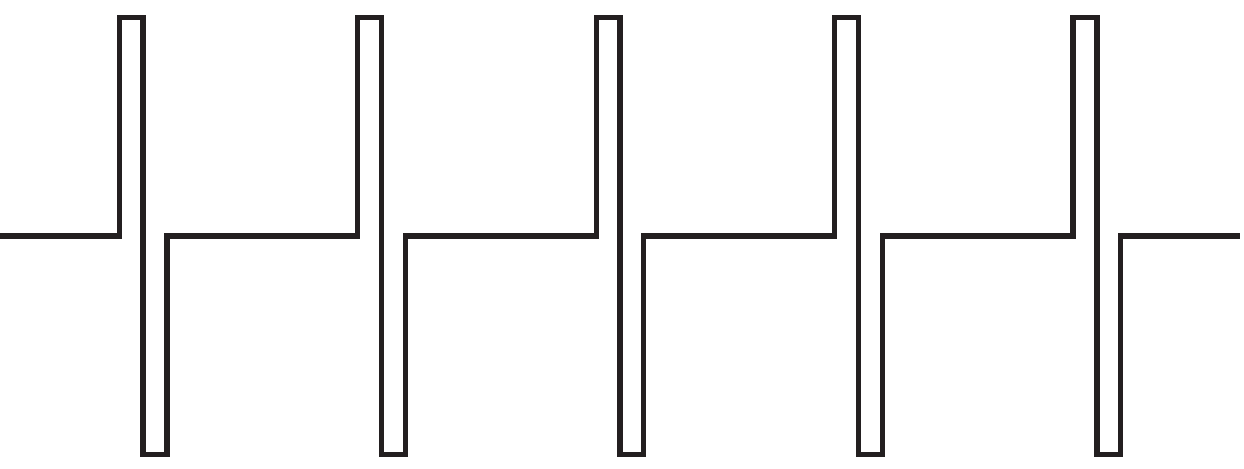}\hspace{2mm}
	  \begin{overpic}[scale=.3]{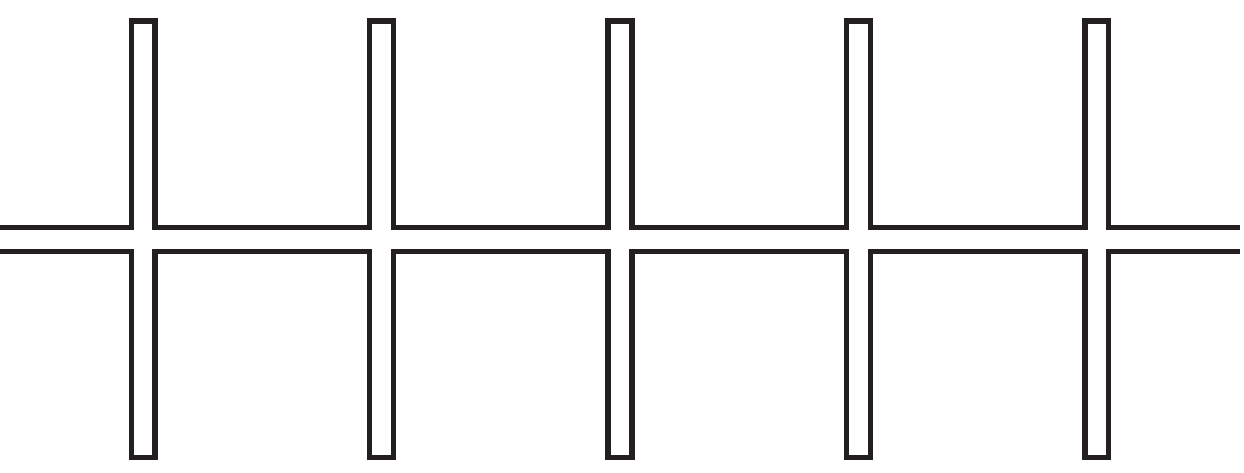}
        \put(-55,-10){(a)}
        \put(49,-10){(b)}
    \end{overpic}
		\vspace*{6mm}
    \caption{(a) Part of the augmented Eulerian path for non-doubled 
		paths. (b) Part of the augmented Eulerian path for doubled paths.}
    \label{fig:path_exp}
		\vspace*{-1mm}
\end{figure}

The Eulerian tour on $T_G$ may have self-intersections, which 
will prevent the tour from being a simple polygon. To address this, we may use 
one of two possible solutions outlined in Fig.~\ref{fig:eliminter} to eliminate
the self-intersections. 
\begin{figure}[ht]
    		\vspace*{2mm}
				\centering
	  \begin{overpic}[width=\columnwidth]{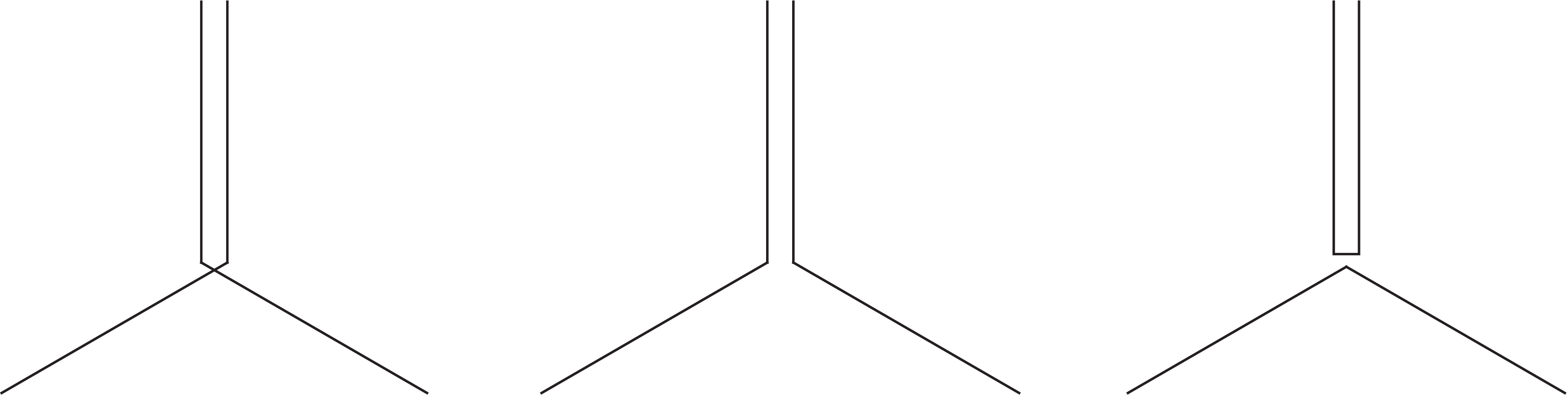}
        \put(12,-6){(a)}
        \put(48,-6){(b)}
        \put(84,-6){(c)}
    \end{overpic}
		\vspace*{3mm}
        \caption{In order to eliminate possible self-intersections in 
				(a), we may transform it into one of the solutions given in 
				(b) and (c) to make the Eulerian tour remain connected (one 
				of the two solutions will satisfy this).}
        \label{fig:eliminter}
\end{figure}

At this point, we readily observe that Theorem~\ref{t:osgthard} applies. 
Furthermore, an optimal solution always allows each mobile sensor to 
cover only two continuous perimeter segments. This is clear in the middle 
of any paths of $T_G$; at junctions, the polygon boundary will be either 
one of two possibilities shown in Fig.~\ref{fig:2types}, where a sensor
again covers at most two continuous segments of the simple polygon. 
\begin{figure}[!ht]
    \centering
    \includegraphics[scale=.29]{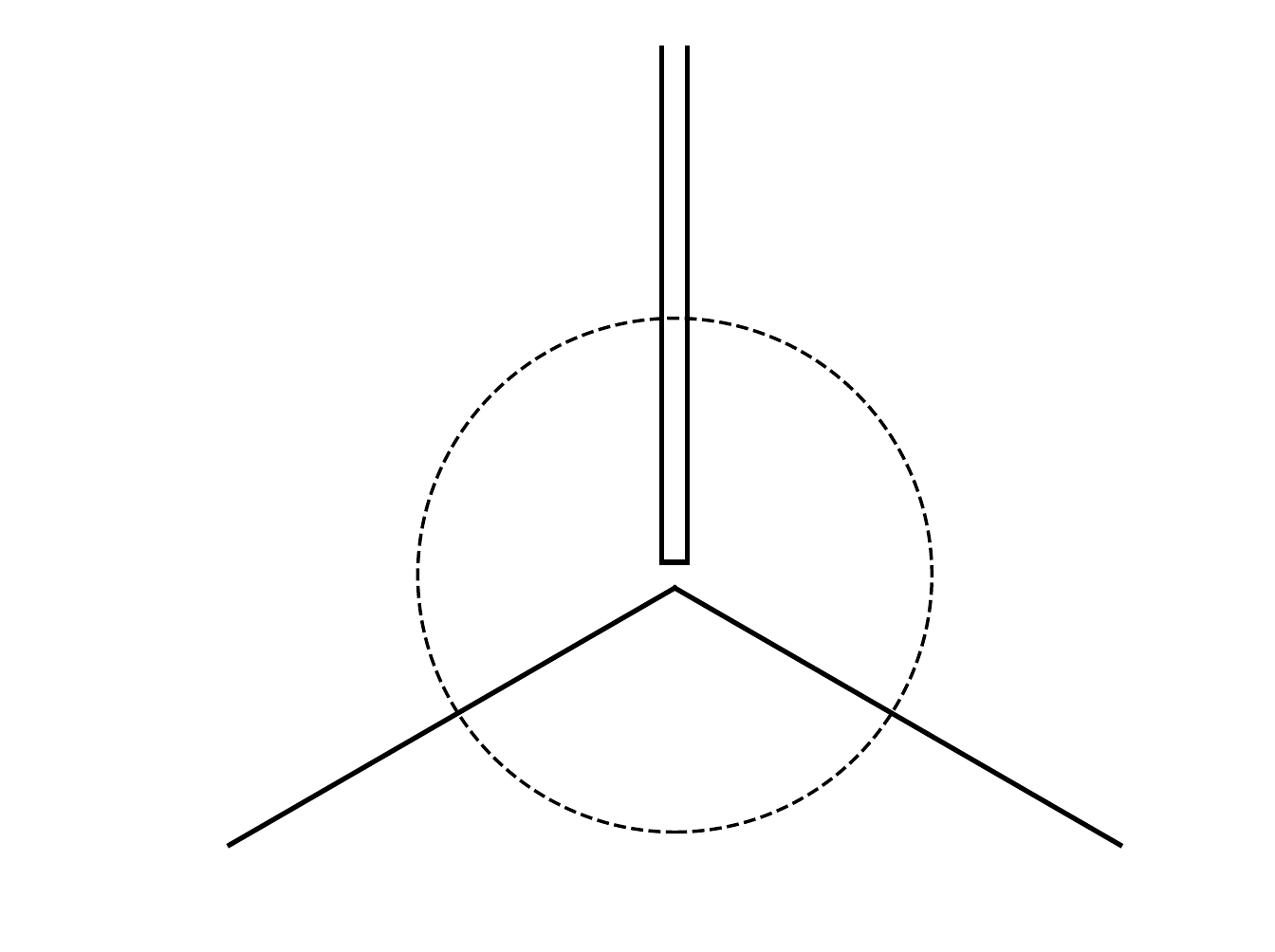}
    \includegraphics[scale=.29]{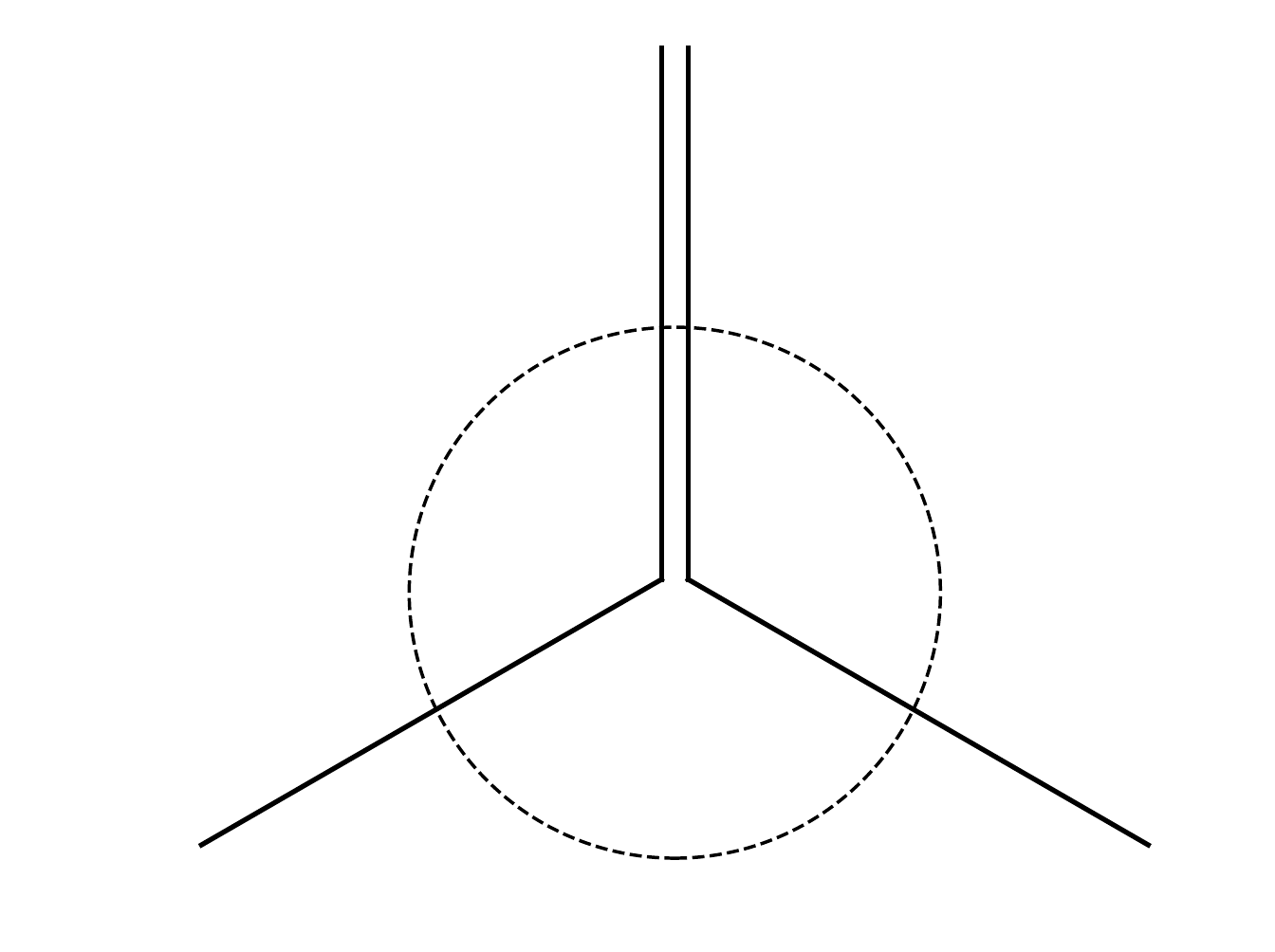}
    \caption{The figure shows two possible types of boundaries near a 
		vertex with degree of $4$. A robot near the vertex will only be able 
		to cover two disjoint but individually continuous boundary segments 
		with sensing radius less than $\alpha$ if the solution is to be optimal.}
    \label{fig:2types}
\end{figure}
\end{proof}

%% file: texs/04-algo.tex
\def\algomed{\textsc{Min\_Enclose\_Disc}\xspace}
In this section, we present several algorithmic solutions for \osgt. 
First, a fully polynomial approximation scheme (FPTAS)
is presented that solve \opgt with the additional requirement 
that each sensor is responsible for a continuous perimeter segment. 
This contrasts Theorem~\ref{them:twoconthard}. Then, we show that 
there exist polynomial time algorithms that readily guarantee a
$(2+\varepsilon)$-approximation for \osgt. This is followed by an 
integer linear programming (ILP) method that delivers high-quality
solutions (as compared with the $(2 + \varepsilon)$-approximate one)
and has good scalability. 

In preparation for introducing the result, we first describe a method
that is used for discretizing the problem. For a simple polygon $P$, 
we can approximately represent its boundary $\partial P$ as a set of 
balls with radius $\varepsilon$ along $\partial P$, by splitting  
$\partial P$ into $N=\lceil{len(\partial P)}/({2\varepsilon})\rceil$ 
continuous pieces of length at most $2\varepsilon$ and putting the balls' 
centers at their midpoints. 
Denote set of $\varepsilon$-balls as $S_B$, and the set of their centers 
as $S_O=\{o_1,\dots, o_N\}$. Since it holds that $size(k,\partial P) \leq 
size(k, S_B) \leq size(k, S_O) + \varepsilon \leq size(k, \partial P) 
+ \varepsilon$, the minimum coverage radius of the discritized version 
of covering $S_O$ will differ no more than $\varepsilon$ from the original 
problem of covering $\partial P$. Similarly, for covering the interior 
of $P$, we can put $P$ into a grid with cell side length $\varepsilon$, 
and set the center of the grid cells intersecting with $P$ as $S_O$, 
creating at most $N=O(({len(\partial P)}/{\varepsilon})^2)$ samples. 
The discretization process converts guarding $P$ or $\partial P$ to 
guarding $S_O$. 

\subsection{\opgt with Single Segment Guarding Limitation}\label{subsec:singleseg}
By Theorem.~\ref{them:twoconthard}, if a mobile sensor can guard up to two 
continuous perimeter segments, \osgt is hard to approximate within 
$1.152$-optimal. 
Translating this into guarding elements of $S_O$, this means that a sensor
can guard two {\em chains} of elements from $S_O$, where each chain contains 
some $m$ elements $o_1, \ldots, o_m$ that are neighbors along $\partial P$. 
Interestingly, if each sensor may only guard a single chain of elements 
from $S_O$, we may compute an optimal cover for $S_O$ using $O(N^2 \log N)$ 
time. This readily turns into a fully polynomial time approximation scheme 
(FPTAS) for \opgt. 
The algorithm operates by checking multiple times whether a given radius 
$r$ is sufficient for $k$ discs of the given radius to cover elements of 
$S_O$ where each disc covers only a single chain of elements. 

A single feasibility check is outlined in Algorithm~\ref{algo:cont}.
In the pseudo code, it is assumed that the indices are modulo $N$, e.g. 
$M[N+1] = M[1]$, $o_{N+1} = o_1$. 
Algorithm~\ref{algo:cont} is based on an efficient implementation of 
a subroutine \algomed (from e.g., \cite{welzl1991smallest,
Mark1997computation}) that computes the disc with minimum radius to 
enclose a given set of points in expected linear time. With this, a sliding 
window can be applied to find the rightmost $end$ for each $1\leq i\leq 
N$ such that $o_i, \dots, o_{end}$ can be enclosed in a circle of radius 
$r$. The length of this sequence is stored in $M[i]$. 

As $o_{end}$ cannot come around and meet $o_i$, the total call to 
\algomed is no more than $2N$. After this, the algorithm simply tries to put 
discs from each $o_i$ to cover as many centers as possible to see 
whether $S_O$ can be enclosed with $k$ discs. An optimization can be made 
by only examining starting point as $o_1, \dots, o_{M[1]+1}$, since there 
is no circle of radius of $r$ that can cover them together by the 
definition of $M$.
The apparent complexity of Algorithm~\ref{algo:cont} is $O(N^2)$. Since 
there are a total of $N$ points and $k$ robots, in a majority 
of cases a circle would enclose about $N/k$ points, which effectively 
lowers the time complexity to $O(N^2/k)$.

\def\algoptfeasi{\textsc{Opg\_2D\_Cont\_Feasible}}
\SetKw{Continue}{continue}
\SetKw{True}{true}
\SetKw{False}{false}
\begin{algorithm}
\begin{small}
\DontPrintSemicolon
\SetKwComment{Comment}{\%}{}
\caption{\textsc{Opg\_2D\_Cont\_Feasible}}    
\label{algo:cont}
\SetKwInOut{Input}{Input}
\SetKwInOut{Output}{Output}
\KwData{$S_O=\{o_1, \dots, o_N\}$, sample points in circular order\\
\qquad\quad$k$, the number of robots\\
\qquad\quad $r$, the candidate sensing radius}
\KwResult{\True or \False, indicating whether $S_O$ can be covered with k discs with radius $r$}
\vspace{2mm}

\If{{\sc Min\_Enclose\_Disc}($o_1, \dots, o_N$)$\leq r$} {\Return{\True}}
\vspace{1.5mm}

\Comment{\small Phase 1: find the maximum number of consecutive points a disc of radius $r$ can enclose
from each $c_i$.}
\vspace{1.5mm}

$M\leftarrow$ an array of length $N$; $end\leftarrow 1$;\;
\vspace{1.5mm}

\For{$i=1$ \KwTo N}{
\While{{\sc Min\_Enclose\_Disc($o_i, \dots, o_{end+1}$) $\leq r$ }}{
    $end\leftarrow end+1$;\;
}
$M[i]\leftarrow end - i + 1$;\;
}
\vspace{1.5mm}

\Comment{\small Phase 2: try to tile from each $o_i$.}
\vspace{1.5mm}

\For{$i=1$ \KwTo $N$}{
$j \leftarrow i$, \quad $cnt \leftarrow k$;\;
\While{$cnt > 0$}{
    $j \leftarrow j + M[j]$;\;
    \If{$j-i \geq N$}{\Return{\True}}
    $cnt\leftarrow cnt-1$\;
}
}
\Return{\False}
\end{small}
\end{algorithm}

Note that for the optimal coverage radius $r^*$, it holds that $r_{min} = 
0 < r^* \leq {len(\partial P)}/({2k}) = r_{max}$. 
Recall that $N=\lceil len(\partial P)/(2\varepsilon)\rceil$.
Hence, after at most 
\[
  \log\frac{r_{max} - r_{min}}{\varepsilon} = \log(\frac{len(\partial P)}{2k\varepsilon})  = O(\log\frac{N}{k} )
\]
times of binary search on the optimal radius $r^*$ by calling \algoptfeasi, 
the search range of $r^*$ or the gap between $r_{max}$ and $r_{min}$ will 
be reduced to within $\varepsilon$. So, it takes expected $O(N^2 \log (N/k))$ 
time in total to get an approximate solution with radius at most $\varepsilon$ 
more than $size(k, S_O)$ or $size(k, \partial P)$.
\begin{theorem}\label{t:opgtsseg}
    Under the rule of continuous coverage, \opgt for a simple 
		polygon can be approximated to $(1 + \varepsilon)$-optimal in expected 
		$O(N^2\log N)$ time, and $O(({N^2}/{k} )\log(N/k))$ in most cases,
    where $N=\lceil{len(\partial P)}/({2\varepsilon})\rceil$.
\end{theorem}

\remark{In the running time complexity analysis, we implicitly 
used the assumption that $len(\partial P)$ is polynomial to problem input 
size (see Section~\ref{sec:problem}). Also, the algorithm given above
computes an $OPT + \varepsilon$ optimal solution. However, 
it can be naturally assumed that the optimal
sensing radius $OPT$ is lower bounded in realistic scenarios. 
So, an $(OPT + \varepsilon)$ solution
directly translates into a $(1 + \varepsilon)$-optimal solution. 
Lastly, using techniques similar to those from \cite{FenHanGaoYuRSS19,FenYu2020RAL}, 
we mention that results in this subsection readily extends to multiple 
simple polygons with gaps along the boundary. 
These arguments continue to apply throughout the rest of  
this section. 

Regarding the choice in implementation, the minimum enclosing disc problem 
(1-center problem) also has deterministic solution 
\cite{megiddo1983linear} in linear time, but a randomized algorithm is considered to be 
more efficient\cite{welzl1991smallest} and easier to implement.}

\subsection{$(2 + \varepsilon)$ Approximation}
In dealing with Euclidean $k$-clustering problems, two seminal methods 
are often brought out, both of which compute $2$-approximation solutions
for $k$-center problem in polynomial time. This is fairly close to the 
inapproximability gap of $1.822$ for Euclidean $k$-center
problem\cite{feder1988optimal}. The first 
\cite{hochbaum1985best, vazirani2013approximation} transforms the 
clustering problem to a dominating set problem and then applies parametric 
search on the cluster size (radius), resulting in a $2$-approximation in 
time $O(n^2\log n)$ with $n$ being the number of points to cover. 
A second method \cite{gonzalez1985clustering} takes 
a simpler farthest clustering approach by iteratively choosing the 
furthest point from the current centers as the new center. The method runs 
in $O(nk)$ but is subsequently improved to $O(n\log k)$ in 
\cite{feder1988optimal}. So, by applying either of them on $S_O$, we have

\begin{proposition}
    \osgt can be approximated to $(2 +\varepsilon)$-optimal in polynomial time 
		with $N=O({len(\partial P)}/{\varepsilon})$ samples for perimeter 
		guarding and $N=O(({len(\partial P)}/{\varepsilon})^2)$ samples for
    region guarding.
\end{proposition}

For evaluation, we implemented the farthest clustering approach 
\cite{gonzalez1985clustering}.

\subsection{Grid and Integer Programming-based Algorithm}
Approximation using grids \cite{har2011geometric} often exhibits good 
optimality guarantees and bounded time complexity. Seeing that and knowing
that \osgt is hard in general, we attempted grid-based integer linear programming 
(ILP) methods for solving \osgt with good success. Our ILP model 
construction is done as follows. 


Consider bounding the polygon $P$ of interest by an $m\times n$ square grid 
where each cell is $\varepsilon \times \varepsilon$, and denote $g_{ij}$ as the center 
of the cell at row $i$ and column $j$. If we limit the possible locations of 
each robot to the center of some grid cell, the optimal radius with this 
limitation will only be at most $\sqrt{2}\,\varepsilon / 2$ away from 
$size(k, S_O)$. This could be seen by moving the robot locations in the optimal deployment
to their nearest grid centers respectively and applying triangle inequality.

So, given a candidate radius $r$, to check the feasibility of whether 
$\partial P$ can be covered by $k$ circles of radius $r$, we adapt an 
approach for solving the $k$-center problem\cite{daskin2000new} with 
integer linear programming. Specifically, we create $m\times n$ boolean 
variables $y_{ij}$, $1\leq i\leq m$, $1\leq j \leq n$, indicating whether 
there is a robot at $g_{ij}$, then start to check the feasibility of 
following integer programming model.
\begin{align}
    \sum_{ 1\leq i\leq m }\,\sum_{1\leq j \leq n} y_{ij} \leq k\qquad \qquad \qquad \\
    \sum_{i,\ j\ s.t.\ \|g_{ij} - o_\ell\|_2\ \leq\ r} y_{ij} \geq 1 \text{ for each }\, 1\leq \ell \leq N\\
    y_{ij} \in \{0,1\}\ \qquad 1\leq i\leq m,\ 1\leq j \leq n
\end{align}
The first constraint says the number of locations is no more than $k$, and 
the second ensures each $o_\ell$ can be covered by at least one circle 
with radius $r$ illustrated in Fig.~\ref{fig:ilpexample}.

\begin{figure}[!ht]
    \centering
		\vspace*{1mm}
		\begin{overpic}[width=1\columnwidth]{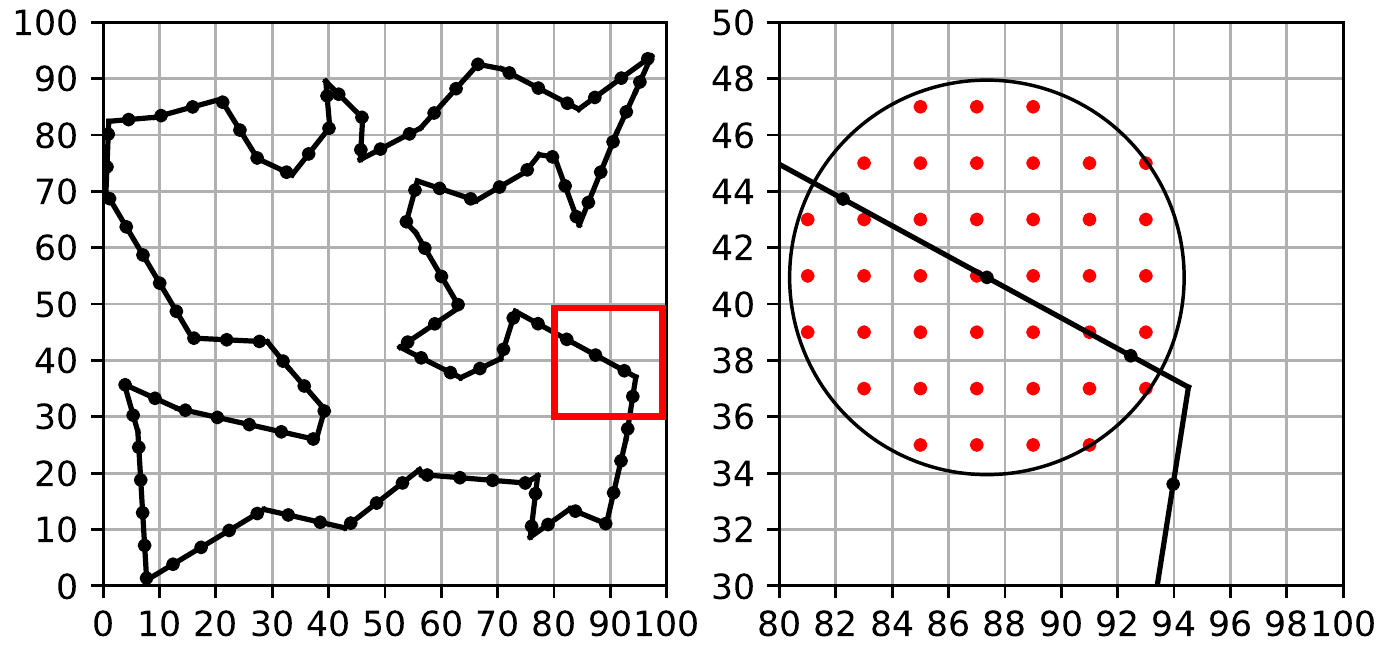}
        \linethickness{0.4mm}
        \put(40.2,25.3){\color{red}\vector(0.8,1.05){15.6}}
        \put(40.2,17){\color{red}\vector(1.45,-1.1){15.5}}
        \put(40,19){\color{blue}$o_{67}$}
        \put(69,24){\color{blue}$o_{67}$}
    \end{overpic}
		\vspace*{1mm}
    \caption{This perimeter guarding example illustrates constraint 
		($2$) for $o_{67}$ with $r=7$. The black dots are the sampled $S_O = 
		\{o_1, \dots, o_{100}\}$. In order to cover $o_{67}$, at least one 
		among the red color grid cell centers need to be selected as robot location.}
    \label{fig:ilpexample}
\end{figure}

When the ILP model has a feasible solution, $r^* = size(k, S_O) \leq r$ 
and $r\leq r^* = size(k, S_O) $ otherwise. This means that we can do a binary 
search on $r^*$, from an initial range of $r^* = size(k, S_O)$: $r_{min}^i=0 
<  r^* \leq {len(\partial P)}/({2k})=r_{max}^i$, until finally $r^f_{max} - 
r^f_{min}$ is reduced to the selected granularity of $\varepsilon$.

\begin{remark}
    With minor modifications, the ILP model applies to 2D region 
		guarding, where the number of constraint $(2)$ will then be $O(mn)$ 
		with one for each grid that intersects with the polygon in an $m\times n$ 
		grid. The initial upper bound set as $r^*$ be $len(\partial P)$ and 
		lower bound set as $\sqrt{{area(P)}/({k\pi})}$. It is also possible to 
		apply the $(2+\varepsilon)$-approximation algorithm and set the result 
		as the initial upper bound with the half of it as the initial lower 
		bound.
\end{remark}

%% file: texs/05-evaluation.tex
For the three algorithms described in Section~\ref{sec:algo}, we developed
implementations in C++ and evaluated them on an Intel Core i7 PC with a 
boost clock of 4.2GHz and 16GB RAM. For solving ILP models, 
Gurobi solver \cite{optimization2019gurobi} is used. 
To evaluate the algorithms, we first generate a set of performance 
benchmarks obtained by subjecting these algorithms through a large set 
of benchmark cases. 
Following the synthetic benchmarks, we applied the algorithms on two 
potential application scenarios: guarding the outer perimeter of the 
Warwick Castle and monitoring a building for potential fire eruption 
points.
 
\subsection{Performance Benchmarks}
For creating synthetic benchmarks, to generate the test set $\W$, we 
created simple polygons with the number of vertices ranging between 
$10$ and $200$. 
For each instance of the tested polygon, vertices are picked uniform at 
random from $[0,1]\times[0,1]$ and the TSP tour among these vertices are 
used for generating a simple polygon of a reasonable shape.
An example is given in Fig.~\ref{fig:ilpexample}.

\def\opgtc{\textsc{Al\_OPG\_2D\_Cont}\xspace}
\def\opgtilp{\textsc{Al\_OPG\_2D\_ILP}\xspace}
\def\orgtilp{\textsc{Al\_ORG\_2D\_ILP}\xspace}
We first evaluate the computational performance of the special \opgt
algorithm where each sensor may cover a single continuous perimeter
segment; denote this algorithm as \opgtc.
Table~\ref{tab:opgcont} lists the running time in seconds for 
various $N$ (number of discretized samples) and $k$ (number of guards). 
Various values of $N$ suggest the choices of $\varepsilon$ according to 
the setup of $N$ in Section~\ref{sec:algo}, in this case 
$N=\lceil {len(\partial P)}/{(2\varepsilon)} \rceil$.
Each data point is an average of $100$ examples. As we can observe, the 
method has very good scalability. It also demonstrates the behavior that 
running time is inverse proportional to the number of guards, 
conforming with the statement about time complexity in Section~\ref{subsec:singleseg}. 
The normalized 
average standard deviation is about $0.06$, which is pretty small. 

\begin{table}[!htbp]
    \centering
    \begin{small}
        \begin{tabularx}{0.49\textwidth}{|X|c|c|c|c|c|c|} 
        \hline
        \diagbox{\hspace{1.4mm}$N$\hspace{0.3mm}}{$k$\hspace{0.7mm}} & $5$ & $10$ & $20$ & $30$ & $50$ & $100$ \\
        \hline
        \hspace{5mm}500&0.097 &        0.044 &        0.019 &        0.013 &        0.007 &        0.004 \\
        \hline
        \hspace{5mm}800&0.257 &       0.118 &       0.054 &        0.036 &        0.019 &        0.011 \\
        \hline
        \hspace{4mm}1000&0.385 &       0.183 &       0.082 &        0.055 &        0.029 &        0.016 \\
        \hline
        \hspace{4mm}1500&0.912 &       0.436 &       0.203 &       0.120 &       0.073 &        0.039 \\
        \hline
        \hspace{4mm}2000&1.597 &      0.743 &       0.345 &       0.225 &       0.123 &       0.062 \\
        \hline        
        \end{tabularx}
    \end{small}
    \vspace{0.1in}
    \caption{
        Running time (seconds) for \opgtc.
    }
    \label{tab:opgcont}
\vspace*{-1mm}
\end{table}

Since the $(2 + \varepsilon)$-optimal algorithm is extremely efficient, 
we do not report its running time. For the ILP methods, Table~\ref{tab:opgilp}
and Table~\ref{tab:orgilp} provide the running times for solving \opgt
\begin{table}[htbp]
    \centering
    \small{
        \begin{tabularx}{0.49\textwidth}{|X|c|c|c|c|c|c|} 
        \hline
        \diagbox{$GS$}{$k$} & $10$ & $15$ & $20$ & $30$ & $50$ & $100$ \\
        \hline
        \hspace{2.2mm}$50\times 50$   &0.219  &0.127  &0.092  &0.051  &0.023  &0.009\\
        \hline
        \hspace{1mm}$100\times 100$ &0.686  &0.383  &0.250  &0.141  &0.089  &0.033 \\ 
        \hline
        \hspace{1mm}$200\times 200$ &1.915         &1.132         &0.792  &0.444  &0.281  &0.115  \\
        \hline
        \hspace{1mm}$300\times 300$ &7.782         &4.201         &2.613         &1.513         &0.814  &0.435 \\
        \hline
        \hspace{1mm}$400\times 400$ &21.23        &11.63        &7.275         &3.827         &2.231         &1.318 \\        \hline
    \end{tabularx}
    }
    \vspace{0.1in}
    \caption{
        Running time (seconds) for \opgtilp.
    }
    \label{tab:opgilp}
\end{table}
and \orgt, respectively (for convenience, denote these two methods as
\opgtilp and \orgtilp).
Each data point is an average over $10$ cases. 
$GS$ denotes the discrete grid size, suggesting the choice of the grid granularity $\varepsilon$
and the single grid cell size $\varepsilon \times \varepsilon$.
We observe that the ILP method is 
highly effective for solving \opgt and fairly good for solving \orgt. 
The normalized average standard deviation is about $0.125$ for \opgtilp
(which is reasonable) and $0.545$ for \orgtilp (which is relatively large). 

\begin{table}[htbp]
    \centering
    \small{
        \begin{tabularx}{0.49\textwidth}{|X|c|c|c|c|c|c|c|} 
        \hline
        \diagbox{$GS$}{$k$} &$10$ & $15$ & $20$ & $30$ & $50$ & $100$ \\
        \hline
        \hspace{2mm}$20\times 20$ &0.252  &0.245  &0.200  &0.170  &0.136  &0.094 \\\hline
        \hspace{2mm}$30\times 30$&1.413 &1.064 &0.886  &0.799  &0.858  &0.576 \\\hline
        \hspace{2mm}$40\times 40$&5.048 &3.598 &3.055 &2.252 &6.114 &1.156 \\\hline
        \hspace{2mm}$50\times 50$ &7.003 &5.617 &4.984 &5.836 &10.91 &0.925\\\hline
        \hspace{2mm}$80\times 80$ &87.14  &84.18   &82.09   &423.5 & $>$2e3 & $>$2e3 \\\hline
        \end{tabularx}
    }
    \vspace{0.1in}
    \caption{
        Running time (seconds) for \orgtilp.
    }
    \label{tab:orgilp}
\end{table}

For solution quality, we compare \opgtc, \opgtilp,
and \orgtilp with the $(2 + \varepsilon)$-optimal solution. For example,
given a test case, let the resulting radius for \opgtc be $r_1$ 
and that for
the $(2+\varepsilon)$-optimal algorithm be $r_2$, we compute the optimality
gain as the reduce of coverage radius over $r_2$ in percentage, 
that is $(r_2 - r_1)/r_2 \cdot 100$. These are then averaged over $10$ cases.
Selected representative results (only three out of a total of 18 rows) are 
given in Table~\ref{tab:comp}. 
In the table, m denotes the method where $\mathbf{1} = $ \opgtc, $\mathbf{2} =$ \opgtilp, 
and $\mathbf{3} = $\orgtilp. Number of samples for \opgtc is set to $2000$. Grid
size for \opgtilp is $200\times 200$. Grid size for \orgtilp is set 
to $40 \times 40$. For each method, we used polygons with $200$ vertices. 
We observe that these algorithms do significantly better than $2$-optimal 
with \opgtilp getting very close to being $1$-optimal (whose optimality gain is
no more than around $50$).

\begin{table}[!htbp]
    \centering
    \small{
        \begin{tabularx}{0.485\textwidth}{|c|c|c|c|c|c|c|} 
        \hline
        \diagbox{\,m}{$k$} & $5$ & $10$ & $20$ & $30$ & $50$ & $100$ \\
        \hline
        $\mathbf{1}$ &\,22.34  &\,23.89  &\,27.07  &\,29.14  &\,32.32  &34.18  \\\hline
        $\mathbf{2}$ &\,36.29  &\,34.82  &\,36.22  &\,36.98  &\,37.69  &38.29  \\\hline
        $\mathbf{3}$ &\,35.69 &\,32.58 &\,30.06 &\,25.22 &\,21.99 &15.46\\\hline
        \end{tabularx}
    }
    \vspace{0.1in}
    \caption{Optimality gain of \opgtc, \opgtilp,
and \orgtilp over the $(2 + \varepsilon)$-optimal method.}
    \label{tab:comp}
\end{table}

\subsection{Two Application Scenarios}
Next, we demonstrate the solutions computed by our algorithms on two 
potential application scenarios. For the first one, we apply algorithms
for \opgt on the outer boundary of the Warwick Castle in England (data
retrieved from openstreetmap.org\cite{haklay2008openstreetmap}). Fig.~\ref{fig:wc} shows the 
solution for $15$ guards computed by the $(2 + \varepsilon)$-optimal 
algorithm, \opgtc, and \opgtilp, respectively. Both \opgtc and \opgtilp 
do about $40\%$ better when compared with the $(2 + \varepsilon)$-optimal 
algorithm. \opgtc does $3\%$ better than \opgtilp since the perimeter is 
suitable for continuous guarding while the ILP method is slightly limited
by the chosen resolution.
\begin{figure}[ht]
    \centering
		\small{
	  \begin{overpic}[width=\columnwidth]{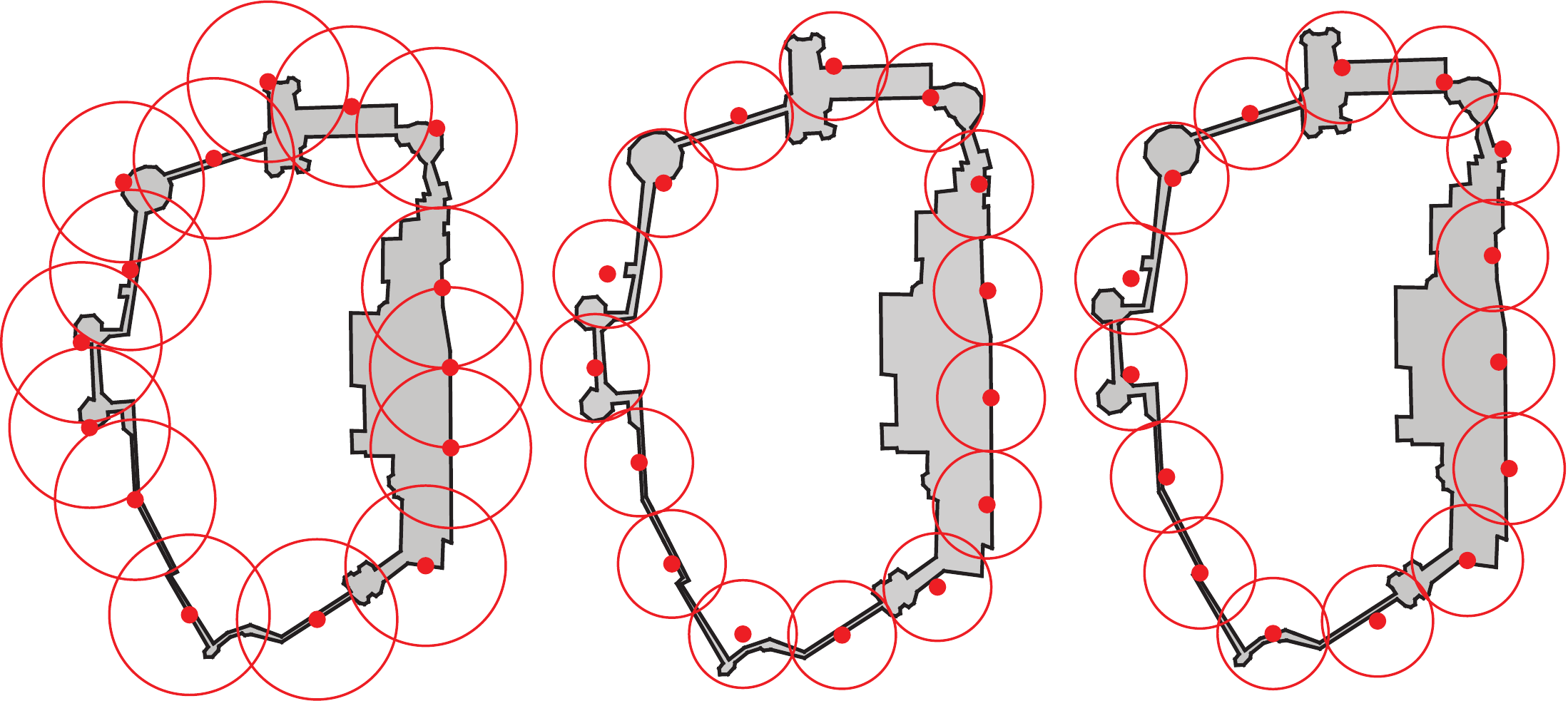}
        \put(17,-3){(a)}
        \put(50,-3){(b)}
        \put(83,-3){(c)}
    \end{overpic}
		}
		\vspace*{1mm}
    \caption{Solutions for deploying $15$ mobile sensors to guard
		the perimeter of the Warwick Castle. Methods: (a) 
		$(2+\varepsilon)$-optimal. (b) \opgtc. (c) \opgtilp. }
    \label{fig:wc}
\end{figure}

In a second application, we took the footprint of the Brazil 
National Museum and use $40$ mobile robots to monitor it. The solution,
shown in Fig.~\ref{fig:museum}, is computed using \orgtilp. This could be
useful when a building is on fire and drones equipped with heat sensors 
can monitoring ``hot spots'' on top of the building to prioritize fire 
extinguishing effort. There are also many other similar application 
scenarios. 

\begin{figure}[ht]
    \centering
    \includegraphics[width=0.95\columnwidth]{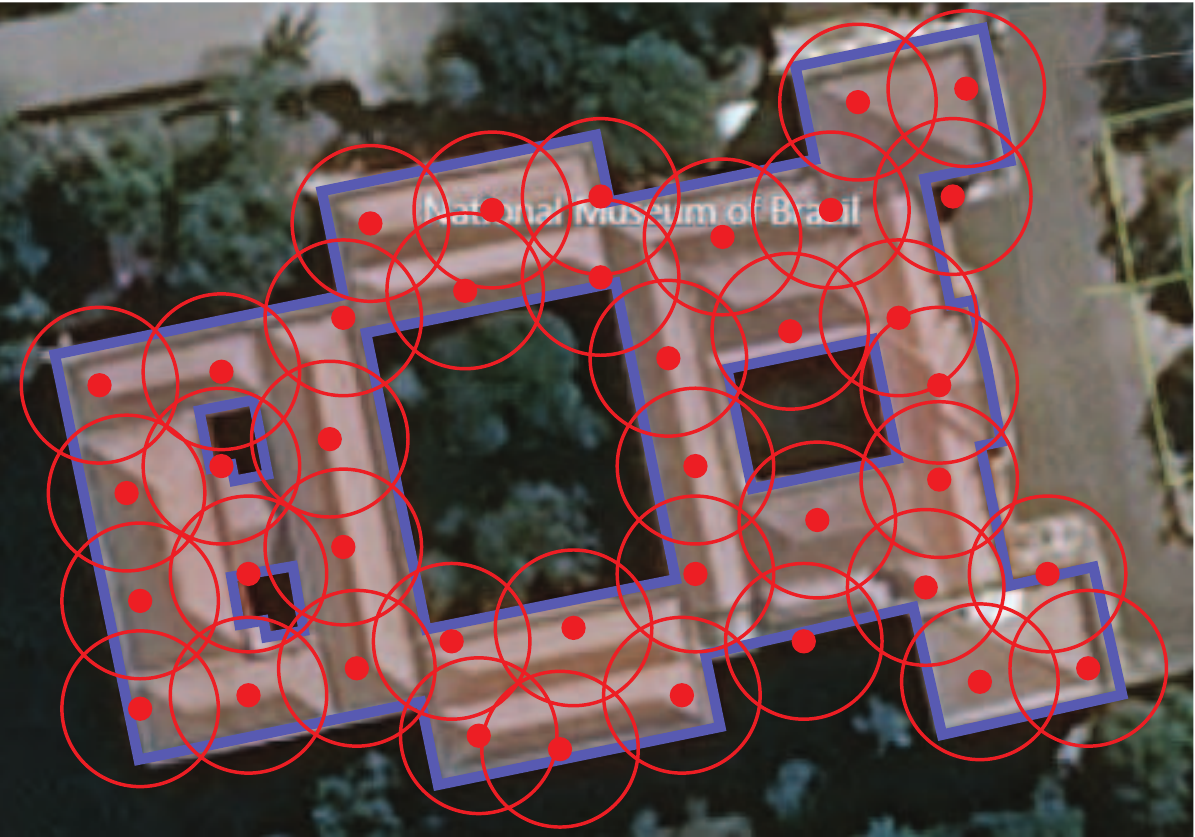}
		\vspace*{1mm}
    \caption{A near-optimal solution for deploying $40$ mobile robots for 
		monitoring the Brazil National Museum, which caught fire in $2019$.}
    \label{fig:museum}
\end{figure}

%% file: texs/06-conclusion.tex
In this study, we examine \osgt, the problem of directly computing a deployment 
strategy for covering 1D or 2D critical sets using many mobile sensors while 
minimizing the maximum sensing radius. After showing that \osgt is computationally 
intractable to even approximate within $1.152$, we describe several 
algorithmic solutions with optimality and/or computation time guarantees. 
Subsequent thorough evaluation demonstrates the effectiveness of these algorithmic 
solutions. Finally, we demonstrate the utility of our algorithmic solutions with 
two application scenarios. 
Due to space limit, guarding perimeters with gaps (see, e.g.,~\cite{FenHanGaoYuRSS19}) 
is not discussed in this work. However, because our algorithms work with a 
grid-based discretization, the results directly apply to arbitrary bounded 1D and 2D 
sets. 

Many intriguing questions follow; we mention two here concerning the sensing 
capabilities. First, \osgt works with circular regions which is perhaps the simplest one
due to symmetry. What if the sensor region is not circular? Whereas such cases 
appear to be hard \cite{culberson1988covering}, effective scalable solutions may 
still be possible. Secondly, currently we assume that all parts of the critical 
set to be guarded have equal importance. What if certain subsets are more important?